\documentclass{article}
\usepackage{fullpage}
\usepackage[ruled,vlined]{algorithm2e}
\usepackage{xcolor}
\usepackage{comment}
\usepackage{metalogo}
\usepackage{enumerate}
\usepackage{mathtools,amssymb,amsmath} 
\usepackage{physics}
\usepackage{makecell}
\usepackage{makeidx}
\usepackage{isomath}
\usepackage[framemethod=TikZ]{mdframed}
\usepackage{enumitem}
\usepackage{relsize}
\usepackage{minitoc}
\usepackage{caption}
\usepackage{graphicx}
\allowdisplaybreaks

\usepackage[square,sort,comma,numbers]{natbib}
\usepackage{amsfonts}
\usepackage{xspace}
\usepackage{array}
\usepackage{multirow}
\usepackage{booktabs}
\usepackage{pgfplots}
\usepackage{pdfpages}
\usepackage{dsfont}
\newcommand{\probset}{\mathcal{P}}

\DeclareMathOperator*{\argmax}{argmax}
\DeclareMathOperator*{\argmin}{argmin}

\DeclareMathOperator*{\probmap}{m}
\newcommand{\probmapset}{\mathcal{M}}

\DeclareMathOperator*{\Risk}{\mathcal{R}}
\DeclareMathOperator*{\Riskemp}{\mathcal{R}_{\mathcal{S}}}
\DeclareMathOperator*{\advRisk}{\mathcal{R}^{adv}}
\DeclareMathOperator*{\advempRisk}{\mathcal{R}^{adv}_\fullSample}
\DeclareMathOperator*{\advRiskzero}{\mathcal{R}^{adv}_{>0}}

\newcommand{\expect}{\mathbb{E}}
\newcommand{\proba}{\mathbb{P}}
\DeclareMathOperator*{\dist}{dist}

\newcommand{\Rademacher}{\mathfrak{R}}
\DeclareMathOperator*{\diam}{diam}

\newcommand{\perturb}{\vectorsym{\tau}}
\newcommand{\inputspace}{ \mathcal{X}}
\newcommand{\inputelement}{\vectorsym{x}}

\newcommand{\R}{\mathbb{R}}
\newcommand{\groundDistrib}{\mathcal{D}}
\newcommand{\outputspace}{\mathcal{Y}}

\newcommand{\Hypothesisspace}{\mathcal{H}}
\newcommand{\hypothesis}{\vectorsym{h}}
\newcommand{\fullSample}{\mathcal{S}}
\newcommand{\equaldef}{:=}
\newcommand{\equationspace}{\thinspace}

\newcommand{\loss}{\mathcal{L}}
\newcommand{\zerooneloss}{\loss_{0/1}}
\newcommand{\zo}{0/1}

\newcommand{\ie}{\emph{i.e.}}
\newcommand{\eg}{\emph{e.g.}}
\newcommand{\aka}{\emph{a.k.a.}}

\newcommand{\st}{\emph{s.t.}}
\newcommand{\wrt}{\emph{w.r.t.}}
\newcommand{\lp}{\ell_p}

\newcommand{\linf}{\ell_\infty}

\newtheorem{proposition}{Proposition}
\newtheorem{theorem}{Theorem}

\newtheorem{corollary}{Corollary}
\newtheorem{lemma}{Lemma}
\newtheorem{definition}{Definition}

\newtheorem{remark}{Remark}
\newtheorem{proof}{Proof}

\title{On the robustness of randomized classifiers to adversarial examples \\
\begin{large}
Adversarial generalization through noise injection
\end{large}
}

\author{
    Rafael Pinot$^{*1}$ \qquad
	Laurent Meunier$^{*2,3}$\qquad Florian Yger$^{2}$\\  
	Cédric Gouy-Pailler$^{4}$ \qquad
	Yann Chevaleyre$^{2}$ \qquad
	Jamal Atif$^{2}$ \qquad
		\vspace{0.3cm}
\\
	$^{1}$ Ecole Polytechnique Fédérale de Lausanne\\
	$^{2}$ LAMSADE, Université Paris-Dauphine\\
	$^{3}$ Facebook AI Research, Paris\\
	$^{4}$ Institut LIST, CEA, Université Paris-Saclay\\}
	
\date{}	

\pgfplotsset{compat=1.16}	
\begin{document}

\maketitle

\begin{abstract}
This paper investigates the theory of robustness against adversarial attacks. We focus on randomized classifiers (\emph{i.e.} classifiers that output random variables) and provide a thorough analysis of their behavior through the lens of statistical learning theory and information theory. 
To this aim, we introduce a new notion of robustness for randomized classifiers, enforcing local Lipschitzness using probability metrics.
Equipped with this definition,  we make two  new contributions. The first one consists in devising a new upper bound on the adversarial generalization gap of randomized classifiers. More precisely, we devise bounds on the generalization gap and the adversarial gap (\emph{i.e.} the gap between the risk and the worst-case risk under attack) of randomized classifiers.  
The second contribution presents a yet simple but efficient noise injection method to design robust randomized classifiers. We show that our results are applicable to a wide range of machine learning models under mild hypotheses. We further corroborate our findings with experimental results using deep neural networks on standard image datasets, namely CIFAR-10 and CIFAR-100. All robust models we trained models can simultaneously achieve state-of-the-art accuracy (over $0.82$ clean accuracy on CIFAR-10) and enjoy \emph{guaranteed} robust accuracy bounds ($0.45$ against $\ell_2$ adversaries with magnitude $0.5$ on CIFAR-10).
\end{abstract}

\section{Introduction}
\label{section::Introduction}
In the last few years, there has been a growing concern on adversarial example attacks in machine learning. An adversarial attack refers to a small (humanly imperceptible) change of an input specifically designed to fool a machine learning model. These attacks have recently come to light thanks to works by~\cite{biggio2013evasion} and~\cite{Szegedy2013IntriguingPO} studying deep neural networks for image classification, although it was an existing topic in spam filter analysis~\cite{dalvi2004adversarial,lowd2005adversarial,globerson2006nightmare}.
The vulnerability of state-of-the-art classifiers to these attacks has genuine security implications especially for deep neural networks used in AI-driven technologies such as self-driving cars, as repetitively demonstrated by~\cite{sharif2016accessorize,sitawarin2018darts} and \cite{selfdrivingattack2020}. Besides security issues, this shows how little we know about the worst-case behaviors of models the industry uses daily. It is essential for the community to understand the very nature of this phenomenon in order to mitigate the threat.

Accordingly, a large body of works has been trying to design new models that would be less vulnerable to the adversarial setting~\cite{goodfellow2014explaining,metzen2017detecting,xie2018mitigating,hu2019new,NIPS2019_9070} but most of them were proven (in time) to offer only limited protection against more sophisticated attacks~\cite{carlini2017adversarial,he2017adversarial,obfuscated-gradients,croce2020reliable,tramer2020adaptive}. Among the defense strategies, randomization has proven effective in some contexts~\cite{Xie2017MitigatingAE,pruningDefenseICLR2018,Xuang2018,rakin2018parametricnoiseinjection}. Albeit these significant efforts, randomization techniques lack theoretical arguments. In this paper, we generalize the prior results from~Pinot et al. \cite{pinot2019theoretical} by studying a general class of randomized classifiers, including randomized neural networks, for which we demonstrate adversarial robustness guarantees and analyze their generalization properties.




\subsection{Supervised learning for image classification in a nutshell} 

Let us consider the supervised classification problem with an input space $\inputspace$ and an output space $\outputspace$. In the following, w.l.o.g. we will consider $\inputspace \subset [-1,1]^d$ to be a set of images, and $\outputspace \equaldef [K] \equaldef \{1,\dots,K\}$ a set of labels describing them. The goal of a supervised machine learning algorithm is to design classifier that maps any image $\inputelement \in \inputspace$ to a label $y \in \outputspace$. To do so, the learner has access to a \emph{training sample} of $n$ image-label pairs $\fullSample\equaldef\{(\vectorsym{x_1},y_1),\dots ,(\vectorsym{x_n},y_n)\}$. Each training pair $(\vectorsym{x_i},y_i)$ is assumed to be drawn \emph{i.i.d.} from a ground-truth distribution $\groundDistrib$. To build a classifier, the usual strategy is to select a hypothesis function $\hypothesis: \mathcal{X} \rightarrow \mathcal{Y}$ from a pre-defined hypothesis class $\Hypothesisspace$ to minimize the \emph{risk} with respect to $\groundDistrib$. This risk minimization problem writes  
\begin{equation}
\label{eq:Initialproblem}
\inf_{\hypothesis \in \Hypothesisspace} \Risk(\hypothesis)  \equaldef \expect_{(\inputelement,y) \sim \groundDistrib}\left[ \zerooneloss\left(\hypothesis(\inputelement), y\right)\right] \equationspace,
\end{equation}
where $\zerooneloss$ represents the $\zo$ loss that outputs $1$ when $\hypothesis(\inputelement) \neq y$, and zero otherwise.

In practice, the learner does not have access to the ground-truth distribution; hence it cannot estimate the risk $\Risk(\hypothesis)$. To find an approximate solution for Problem~\eqref{eq:Initialproblem}, a learning algorithm solves the \emph{empirical risk minimization} problem instead. In this case, we simply replace the risk by its empirical counterpart over the training sample $\fullSample\equaldef\{(\vectorsym{x_1},y_1),\dots ,(\vectorsym{x_n},y_n)\}$. The empirical risk minimization problem writes
\begin{equation}
\inf_{\hypothesis \in \Hypothesisspace} \Riskemp(\hypothesis) \equaldef \frac1n \sum_{i=1}^{n} \zerooneloss\left(\hypothesis(\vectorsym{x_i}), y_i\right)\equationspace.\label{eq:EmpInitialproblem}
\end{equation}
Then, to evaluate how far the selected hypothesis is from the optimum, one wants to upper bound the difference between the risk and the empirical risk of any $\hypothesis \in \Hypothesisspace$. This difference is known as the \emph{generalization gap}. 

\subsection{Classification in the presence of an adversary}
Given a hypothesis $\hypothesis \in \Hypothesisspace$ and a sample $(\inputelement,y) \sim \groundDistrib$, the goal of an adversary is to find a perturbation $\vectorsym{\tau} \in \inputspace$ such that the following assertions \emph{both} hold. First, the perturbation is imperceptible to humans. This means that a human cannot visually distinguish the standard example $\inputelement$ from the \emph{adversarial example} $\inputelement + \perturb$. Second, the perturbation modifies $\inputelement$ enough to make the classifier misclassify. More formally, the adversary seeks a perturbation $\vectorsym{\tau} \in \inputspace$ such that $\hypothesis(\inputelement+\vectorsym{\tau}) \neq y$.

Although the notion of imperceptible modification is very natural for humans, it is genuinely hard to formalize. Despite these difficulties, in the image classification setting, a sufficient condition to ensure that the attack will remain undetected is to constrain the perturbation $\vectorsym{\tau}$ to have a small $\ell_p$ norm. This means that for any $p \in [1,\infty]$, there exists a threshold $\alpha_p > 0$ for which any perturbation $\perturb$ is imperceptible as soon as $\norm{\perturb}_p \leq \alpha_p$. The literature on adversarial attacks for image classification usually uses either an $\ell_\infty$ norm akin~\cite{madry2017towards} or an $\ell_2$ norm akin~\cite{carlini2017adversarial} as a surrogate for imperceptibility. Other authors such as~\cite{chen2018ead} and  \cite{papernot2016distillation} also used an $\ell_1$ norm or an $\ell_0$ semi-norm.

To account for adversaries possibly manipulating the input images, one needs to revisit the standard risk minimization by incorporating the adversary in the problem. The goal becomes to minimize the \emph{worst-case} risk under $\alpha_p$-bounded manipulations. We call this problem the \emph{adversarial risk minimization}. It writes
\begin{equation}
\inf_{\hypothesis \in \mathcal{H}} \advRisk(\hypothesis; \alpha_p) \equaldef \expect_{(\inputelement,y)\sim \mathcal{D}}\left[  \sup_{ \perturb \in B_p(\alpha_p)}\zerooneloss\left(\hypothesis(\inputelement + \vectorsym{\tau}), y\right)\right]\equationspace, \label{eq:Advproblem}
\end{equation}
where $B_p(\alpha_p) \equaldef\{ \tau \in \inputspace ~\st~ \norm{\perturb}_p \leq \alpha_p\}$. In this new formulation, the adversary focuses on optimizing the inner maximization, while the learner tries to get the best hypothesis from $\mathcal{H}$ ``under attack’’. By analogy with the standard setting, given $n$ training examples $\mathcal{S}:=\{(\vectorsym{x_1},y_1),\dots ,(\vectorsym{x_n},y_n)\}$, we want to find an approximate solution to the adversarial risk minimization by studying its empirical counterpart, the \emph{empirical adversarial risk minimization}. This optimization problem writes
\begin{equation}
\inf_{\hypothesis \in \Hypothesisspace} \advempRisk(\hypothesis;\alpha_p) \equaldef \frac1n\sum_{i=1}^{n} \sup_{\perturb \in B_p(\alpha_p)}\zerooneloss\left(\hypothesis(\vectorsym{x_i} + \vectorsym{\tau}), y_i\right)\equationspace. \label{eq:EmpAdvproblem}
\end{equation}
In the presence of an adversary, two major issues appear in the empirical risk minimization. First, as recently pointed out by \cite{madry2017towards}, the adversarial generalization error (\emph{i.e.} the gap between the empirical adversarial risk and the adversarial risk) can be much larger than in the standard setting. Indeed, the adversary makes the problem dependent on the dimension of $\inputspace$. Hence, in high-dimension (\emph{e.g.} for images) one needs much more samples to classify correctly as pointed out by~\cite{schmidt2018adversarially} as well as \cite{simon2019first}. Moreover, finding an approximate solution to the adversarial risk minimization is not always sufficient. Indeed, recent works by~\cite{tsipras2018robustness} and \cite{zhang2019theoretically} gave theoretical evidence that training a robust model may lead to an increase of its standard risk. Hence finding a good approximation for Problem~\eqref{eq:Advproblem} may lead to a poor solution for Problem~\eqref{eq:Initialproblem}. Accordingly, it is natural to wonder whether we can \textbf{\emph{find a class of models $\boldsymbol{\Hypothesisspace}$ for which we can control both the standard and adversarial risks?}}

In this paper, we provide answers to the above question by conducting an in depth analysis of a special class of models called randomized classifiers, \emph{i.e.} classifiers that output random variables instead of labels. Our main contributions summarize as follows. 

\subsection{Contributions}
\label{sec:contrib}

Our first contribution consists in studying randomized classifiers. By analogy with the deterministic case, we define a notion of robustness for randomized classifiers. This definition amounts to making the classifier locally Lipschitz with respect to the $\ell_p$ norm on $\inputspace$, and a probability metric on $\outputspace$ (\emph{e.g.} the total variation distance or the Renyi divergence). More precisely, if we denote $D$ the probability metric at hand, a randomized classifier $\probmap$ is called $(\alpha_p, \epsilon)$-robust \wrt~$D$ if for any $\inputelement,\inputelement' \in \mathcal{X}$
$$ \norm{\inputelement - \inputelement'}_p \leq \alpha_p \implies D(\probmap(\inputelement),\probmap(\inputelement')) \leq \epsilon.$$
Denoting $\probmapset_D(\alpha_p,\epsilon)$ the class of randomized classifiers that respect this local Lipschitz condition, we present the following results. 
\begin{enumerate}
    \item If $D$ is either the total variation distance or the Renyi divergence, we show that for any $\probmap \in \probmapset_D(\alpha_p,\epsilon)$, we can upper-bound the gap between the risk and the adversarial risk of $\probmap$. Notably, if $D$ is the total variation distance, for any $\probmap \in \probmapset_D(\alpha_p,\epsilon)$ we have $ \advRisk(\probmap;\alpha_p) - \Risk(\probmap) \leq \epsilon$. Hence, $\epsilon$ controls the maximal trade-off between robust and standard accuracy for locally Lipschitz randomized classifier. We demonstrate similar results when $D$ is the Renyi divergence showing that $\advRisk(\probmap;\alpha_p) - \Risk(\probmap) \leq 1- O\left(e^{-\epsilon}\right)$. This means that, for the class of locally Lipschitz randomized classifiers, solving the risk minimization problem, \ie~Problem~\eqref{eq:Initialproblem}, gives an approximate solution to the adversarial risk minimization problem, \ie~Problem~\eqref{eq:Advproblem}, up to an additive factor that depends on the robustness parameter $\epsilon$. 
    
    \item We devise an upper-bound on the generalization gap of any $\probmap$ in $\probmapset_D(\alpha_p,\epsilon)$. In particular, when $D$ is the total variation distance, we demonstrate that for any $\probmap \in \probmapset_D(\alpha_p,\epsilon)$ we have $$\Risk(\probmap) - \Riskemp(\probmap) \leq O\left(\sqrt{\frac{N \times K}{n}}\right) + \epsilon,$$ where $N$ is the external $\alpha_p$-covering number of the input samples. This means that, when $N/n \underset{n \rightarrow \infty}{\rightarrow} 0$, solving the empirical risk minimization problem, \ie~Problem~\eqref{eq:EmpInitialproblem}, on $\probmapset_D(\alpha_p,\epsilon)$ provides an approximate solution to the risk minimization problem, \ie~Problem~\eqref{eq:Initialproblem}. Since we can also bound the gap between the adversarial and the standard risk, we can combine the two results to bound the adversarial generalization gap on $\probmapset_D(\alpha_p,\epsilon)$. Note however, that this result relies on a strong assumption on $\inputspace$ that does not always avoid dimensionality issues. The problem of finding a subclass of $\probmapset_D(\alpha_p,\epsilon)$ that provides tighter generalization bounds is an open question.
\end{enumerate}

For our second contribution, we present a practical way to design this class $\probmapset(\alpha_p,\epsilon)$ by using a simple yet efficient noise injection scheme. This allows us to build randomized classifiers from state-of-the-art machine learning models, including deep neural networks. More precisely our contribution is as follows.

\begin{enumerate}
    \item Based on information-theoretic properties of the total variation distance and the Renyi divergence (\eg~the data processing inequality) we design a noise injection scheme to turn a state-of-the-art machine learning model into a robust randomized classifier. More formally, Let us denote $\Phi$ the c.d.f. of a standard Gaussian distribution. Let us consider $\hypothesis$ a deterministic hypothesis, we show that the randomized classifier $\probmap: \inputelement \mapsto \hypothesis\left(\inputelement+n\right)$ with $n\sim\mathcal{N}(0, \sigma^2 I_d)$ is both $(\alpha_2, \frac{(\alpha_2)^2}{2 \sigma})$-robust \wrt~the Renyi divergence and $(\alpha_2,\ 2 \Phi\left( \frac{\alpha_2}{2 \sigma} \right) - 1)$-robust \wrt~the total variation distance. Our results on randomized classifiers are applicable to a wide range of machine learning models including deep neural networks.
    
    \item We further corroborate our theoretical results with experiments using deep neural networks on standard image datasets, namely CIFAR-10 and CIFAR-100~\cite{krizhevsky2009learning}. These models can simultaneously provide accurate prediction (over $0.82$ clean accuracy on CIFAR-10) and reasonable robustness against $\ell_2$ adversarial examples ($0.45$ against $\ell_2$ adversaries with magnitude $0.5$ on CIFAR-10). 
\end{enumerate}

\section{Related Work}
\label{section::RW}

Contrary to other notions such as training corruption, \aka~poisoning attacks~\cite{kearns1993learning,kearns1994toward}, the theoretical study of adversarial robustness is still in its infancy. So far, empirical observations tend to show that 1) adversarial examples on state-of-the-art models are hard to mitigate and 2) robust training methods give poor generalization performances. Some recent works started to study the problem through the lens of learning theory either to understand the links between robustness and accuracy or to provide bounds on the generalization gap of current learning procedures in the adversarial setting.

\subsection{Accuracy vs robustness trade-off}
\label{sec:AccRobTradeoff}
 
A first line of research~\cite{su2018robustness,10.5555/3327546.3327734,tsipras2018robustness} suggests that designing robust models might be inconsistent with standard accuracy. These works argue with experiments and toy examples that robust and standard classification are two concurrent problems. Following this line, \cite{zhang2019theoretically} observed that the adversarial risk of any hypothesis $\hypothesis$ decomposes as follows,
\begin{equation}
\advRisk(\hypothesis;\alpha_p) =  \Risk(\hypothesis) +  \advRiskzero(\hypothesis;\alpha_p),
\label{eq:decomposition}
\end{equation} 
where $\advRiskzero(\probmap;\alpha_p)$ is the amount of risk that the adversary gets with \emph{non-null} perturbations. Looking at Equation~\eqref{eq:decomposition}, we realize that minimizing the adversarial risk is not enough to control standard accuracy, as one could only optimize over the second term. This indicates that adversarial risk minimization, \ie~Problem~\eqref{eq:Advproblem}, is harder to solve than the standard risk minimization, \ie~Problem~\eqref{eq:Initialproblem}. 

While this indicates that both goals maybe difficult be achieve simultaneously, Equation~\eqref{eq:decomposition}, along with the empirical studies from the literature  do not highlight any fundamental trade-off between robustness and accuracy. Moreover, no upper-bound on $\advRiskzero(\hypothesis;\alpha_p)$ has been demonstrated yet. Hence the questions whether this trade-off exists and can be controlled remain open. In this paper, we provide a rigorous answer to these questions by identifying classes $\probmapset_D(\alpha_p,\epsilon)$ of randomized classifiers for which we can upper bound the trade-off term $\advRiskzero(\probmap ;\alpha_p)$ for any $\probmap \in \probmapset_D(\alpha_p,\epsilon)$. This shows that for some classes of randomized classifiers, precision is not conflicting with robustness, since we can control the maximum loss of accuracy that the model can suffer in the adversarial setting. It also challenges the intuitions developed by previous works~\cite{su2018robustness,10.5555/3327546.3327734,tsipras2018robustness} and argues in favor of using randomized mechanisms as a defense against adversarial attacks.

\subsection{Studying adversarial generalization}

To further compare the hardness of the two problems, a recent line of research began to explore the notion of adversarial generalization gap. In this line, \cite{schmidt2018adversarially} presented some first intuitions by studying a simplified binary classification framework where $\groundDistrib$ is a mixture of multi-dimensional Gaussian distributions. In this framework the authors show that without attacks,  we only need $O(1)$ training samples to have a small generalization gap.
But against an $\linf$ adversary, we need $O(\sqrt{d})$ training samples instead. In the discussion of their work, the authors present the problem of obtaining similar results without making any assumption about the distribution as an open problem.

This issue was recently studied using the Rademacher complexity by \cite{khim2018adversarial,yin2019rademacher} and \cite{awasthi2020adversarial}. These papers relate the adversarial generalization error of linear classifiers and one-hidden layer neural networks with the dimension of the problem. They show that the adversarial generalization depends on the dimension of the problem. At a first glance, the difficulty of adversarial generalization seems to contradict previous conclusions on the link between robustness and generalization presented by~\cite{xu2012robustness}. But, as we will discuss in the sequel, these results assume that the input space $\mathcal{X}$ can be partitioned in $O(1)$ sub-space in which the classification function has small variations. This assumption may not always hold when dealing with high dimensional input spaces (\eg~images) and very sophisticated classification algorithms (\eg~deep neural networks).

Going further, it should be noted that the generalization gap
measures only the difference between empirical and theoretical risks. In practice, the empirical adversarial risk is hard to estimate, since we cannot compute the exact solution to the inner maximization problem. The following question therefore remains open: even if we can set up a learning procedure with a controlled generalization gap, can we give guarantees on the standard and adversarial risks? In this paper, we start answering this question by providing techniques that provably offer both small standard risk and reasonable robustness against adversarial examples (see Section~\ref{sec:contrib} for more details).


\subsection{Defense against adversarial examples based on noise injection}

Injecting noise into algorithms to improve train time robustness has been used for ages in detection and signal processing tasks~\cite{ZozoA99,ChapR04,MitaK98,grandvalet1997noise}. It has also been extensively studied in several machine learning and optimization fields, \eg~robust optimization~\cite{ben2009robust} and data augmentation techniques~\cite{Perez2017TheEO}. Concurrently to our work, noise injection techniques have been adopted by the adversarial defense community under the \emph{randomized smoothing} name. The idea of provable defense through noise injection was first proposed by \cite{lecuyer2019certified} and refined by \cite{li2019certified,KolterRandomizedSmoothing} and \cite{salman2019provably}. The rational behind randomized smoothing is very simple: smooth $\hypothesis$ \emph{after training} by convolution with a Gaussian measure to build a more stable classifier. Our work belongs to the same line of research, but the nature of our results is different. While randomized smoothing focuses on the construction of certified defenses, depending on the dataset and the classifier at hand, we study the generalization properties of randomized mechanisms both in the standard and the adversarial setting. Our analysis presents the fundamental properties of randomized defenses, including (but not limited to) randomized smoothing (c.f. Section~\ref{sec:modepreservationendRS}).

\section{Definition of Risk and Robustness for Randomized classifiers}
\label{section::RiskforRandomClassifiers}
In this work, the goal is to analyze how randomized classifiers can solve the problem of classification in the presence of an adversary. Let us start by defining what we mean by randomized classifiers.

\begin{remark}[Remark on measurability] 
Through the paper, we assume every spaces $\mathcal{Z}$ to be associated with a $\sigma$-algebra denoted $\mathcal{A}\left( \mathcal{Z}\right)$. Furthermore, we denote $\probset\left(\mathcal{Z} \right)$ the set of probability distributions defined on the measurable space $\left(\mathcal{Z},\mathcal{A}\left(\mathcal{Z}\right)\right)$. In the following, for simplicity, we  refer to $\mathcal{A}\left(\mathcal{Z}\right)$ only when necessary.
\end{remark}

\begin{definition}[Probabilistic mapping]
\label{def::ProbMapping}
Let $\mathcal{Z}$ and $\mathcal{Z}'$ be two arbitrary spaces. A \emph{probabilistic mapping} from $\mathcal{Z}$ to $\mathcal{Z}'$ is a mapping $\probmap: \mathcal{Z} \rightarrow \probset\left(\mathcal{Z}' \right)$, where $ \probset\left(\mathcal{Z}' \right)$ is the space of probability measures on $\mathcal{Z}'$. 
When $\mathcal{Z} = \inputspace$ and $\mathcal{Z}' =\mathcal{Y}$, $\probmap$ is called a \emph{randomized classifier}. To get a numerical answer out of $\probmap$ for an input $\inputelement$, we sample $\hat{y} \sim \probmap( \inputelement )$.
\end{definition}

Any mapping can be considered as a probabilistic mapping, whether it explicitly considers randomization or not. In fact, any deterministic classifier can be considered as a randomized one, since it can be characterized by a Dirac measure. Accordingly, the definition of a randomized classifier is fully general and equally consider classifiers with or without randomization scheme.

\subsection{Risk and adversarial risk for randomized classifiers}

To analyze this new hypothesis class, we can adapt the concepts of risk and adversarial risk for a randomized classifier. The loss function we use is the natural extension of the $\zo$ loss to the randomized regime. Given a randomized classifier $\probmap$ and a sample $(\inputelement,y) \sim \groundDistrib$ it writes
\begin{align}
    \zerooneloss(\probmap(\inputelement),y) := \expect_{\hat{y} \sim \probmap(\inputelement)}  \left[ \mathds{1} \left\{\hat{y} \neq y\right\}\right] .
\end{align}
This loss function evaluates the probability of misclassification of $\probmap$ on a data sample $(\inputelement,y) \sim \groundDistrib$. Accordingly, the risk of $\probmap$ with respect to $\groundDistrib$ writes
\begin{align}
\Risk(\probmap) &:=  \expect_{(\inputelement,y)\sim \groundDistrib}\left[ \zerooneloss(\probmap( \inputelement),y)   \right].
\end{align}
Finally, given $\probmap$ and $(\inputelement,y) \sim \groundDistrib$, the adversary seeks a perturbation $\perturb \in  B_p(\alpha_p)$ that maximizes the expected error of the classifier on $\inputelement$ (\emph{i.e.} $\expect_{\hat{y} \sim \probmap(\inputelement + \perturb)}  \left[ \mathds{1} \left\{\hat{y} \neq y\right\}\right]$). Therefore, the adversarial risk of $\probmap$ under $\alpha_p$-bounded perturbations writes
\begin{align}
\advRisk(\probmap;\alpha_p) &:= \expect_{(\inputelement,y)\sim \groundDistrib}\left[ \sup_{  \perturb \in B_p(\alpha_p)} \zerooneloss(\probmap(\inputelement + \perturb),y)  \right].
\end{align}

By analogy with the deterministic setting, we denote $\Riskemp\left( \probmap \right) \equaldef \frac1n\sum_{i=1}^n \zerooneloss \left(\probmap(\inputelement_i), y_i\right)$ and $\advempRisk \left( \probmap ; \alpha_p \right):= \frac1n\sum_{i=1}^n \sup_{\perturb \in B_p(\alpha_p)}\zerooneloss \left(\probmap(\vectorsym{x_i} + \perturb), y_i\right)$ the empirical risks of $\probmap$ for a given training sample $\mathcal{S}:=\{ (\vectorsym{x_1},y_1), \dots , (\vectorsym{x_n},y_n) \}$. 

\subsection{Robustness for  randomized classifiers} 
We could define the notion of robustness for a randomized classifier depending on whether it misclassifies any test sample $(\inputelement,y) \sim \groundDistrib$. But in practice, neither the adversary nor the model provider have access to the ground-truth distribution $\groundDistrib$. Furthermore, in real-world scenarios, one wants to check before its deployment that the model is robust. Therefore, it is required for the classifier to be stable on the regions of the space where it already classifies correctly. Formally a (deterministic) classifier $c: \inputspace \rightarrow \mathcal{Y}$ is called \emph{robust} if for any $(\inputelement, y) \sim \groundDistrib$ such that $c(\inputelement) = y$, and  for any $\perturb \in \inputspace $ one has  
\begin{equation}
    \norm{ \perturb }_p \leq \alpha_p \implies c(\inputelement) = c(\inputelement + \perturb ). 
\end{equation}
By analogy with this notion, we define robustness for a randomized classifier as follows.
\begin{definition}[Robustness for a randomized classifier]
A randomized classifier $\probmap: \inputspace \rightarrow \probset(\mathcal{Y})$ is called $(\alpha_p,\epsilon)$-\emph{robust} w.r.t. $D$ if for any $\inputelement, \perturb \in \inputspace$, one has
\begin{align*}
 \norm{\perturb}_p \leq \alpha_p \implies D\left(\probmap(\inputelement)  ,  \probmap(\inputelement + \perturb)\right) \leq \epsilon \equationspace.   
\end{align*}
Where $D$ is a metric/divergence between two probability measures. Given such a metric/divergence $D$, we denote $\probmapset_{D}(\alpha_p,\epsilon)$ the set of all randomized classifiers that are $(\alpha_p,\epsilon)$-\emph{robust} w.r.t. ~$D$.
\end{definition}

Note that we did not add the constraint that $\probmap$ classifies well on $(\inputelement,y) \sim \groundDistrib$, since it is already encompassed in the probability distribution itself. If the two probabilities $\probmap(\inputelement)$ and $\probmap(\inputelement + \perturb)$ are close, and if $\probmap(\inputelement)$ outputs $y$ with high probability, then it will be the same for $\probmap(\inputelement + \perturb)$. This formulation naturally raises the question of the choice of the metric $D$. Any choice of metric/divergence will instantiate a notion of adversarial robustness, and it should be carefully selected. In the present work, we focus our study on the total variation distance and the Renyi divergence. The question whether these metrics/divergences are more appropriate than others remains open but these two divergences are sufficiently general to cover a wide range of other definitions (see Appendix~\ref{appendix::discussion} for more details). Furthermore, these notions of distance comply with both a theoretical analysis (Section~\ref{section::GeneralizationBoundAdvGap}) and practical considerations (Section~\ref{section::Experiments}).

\subsection{Divergence and metrics between probability measures.}
\label{section::Preliminaries}

Let us now recall the definition of total variation distance and Renyi divergence. Let $\mathcal{Z}$ be an arbitrary space, and $\rho$, $\rho'$ be two measures in $\probset(\mathcal{Z})$\footnote{Recall from Definition~\ref{def::ProbMapping} that $\probset(\mathcal{Z})$ is the set of probability measures on $\mathcal{Z}$}. 
The \emph{total variation distance} between $\rho$ and $\rho' $ is
\begin{align} D_{TV}\left(\rho  , \rho' \right) := \sup\limits_{Z \subset \mathcal{A} (\mathcal{Z})} |\rho (Z) - \rho' (Z)| \enspace,
\end{align}
where $\mathcal{A}(\mathcal{Z})$ is the $\sigma$-algebra associated with the set of measures $\probset(\mathcal{Z})$.
The total variation distance is one of the most commonly used probability metrics. It admits several very simple interpretations, and is a very useful tool in many mathematical fields such as probability theory, Bayesian statistics or optimal transport~\cite{villani2003topics,robert2007bayesian,peyre2019computational}. In optimal transport, it can be rewritten as the solution of the Monge-Kantorovich problem with the cost function $\text{cost}(\vectorsym{z},\vectorsym{z}') =\mathds{1}\left\{ \vectorsym{z}\neq \vectorsym{z}'\right\}$,
\begin{equation}
    D_{TV}(\rho  , \rho' ) = \inf\int_{\mathcal{Z}^{2}}\mathds{1}\left\{ \vectorsym{z} \neq \vectorsym{z}'\right\} d\pi(\vectorsym{z},\vectorsym{z}') \enspace,
\end{equation} 
where the infimum is taken over all joint probability measures $\pi$ in $\probset\left( \mathcal{Z}\times\mathcal{Z} \right)$ with marginals $\rho$ and $\rho' $. According to this interpretation, it seems quite natural to consider the total variation distance as a relaxation of the trivial distance on $[0,1]$ (for deterministic classifiers). 

Let us now suppose that $\rho$ and $\rho'$ admit probability density functions $g$ and $g'$ according to a third measure $\nu$. Then the \emph{Renyi divergence of order $\beta$} between $\rho$ and $\rho'$ writes
\begin{align}
D_{\beta}\left(\rho  , \rho' \right):=\cfrac{1}{\beta -1}\log \int_{\mathcal{Y}} g' (y)  \left(\cfrac{g(y)}{g' (y)}\right)^{\beta} d\nu(y)\enspace.
\end{align}
The Renyi divergence~\cite{renyi1961} is a generalized divergence defined for any $\beta$ on the interval $[1,\infty]$.  It equals the Kullback-Leibler divergence when $\beta \rightarrow 1$, and the maximum divergence when $\beta \rightarrow \infty$. It also has the property of being non-decreasing with respect to $\beta$. This divergence is very common in machine learning and Information theory~\cite{6832827}, especially in its Kullback-Leibler form as it is widely used as the loss function, \ie~cross entropy, of classification algorithms. In the remaining, we denote $\probmapset_{\beta}\left(\alpha_p,\epsilon\right)$ the set of $(\alpha_p,\epsilon)$-robust classifiers w.r.t. $D_{\beta}$.

Let us now give some properties of these divergences that will be useful for our analysis. First we recall the probability preservation property of the Renyi divergence, first presented by~\cite{langlois2014gghlite}.
\begin{proposition}[\cite{langlois2014gghlite}] 
\label{prop::renyi}
Let $\rho$ and $\rho' $ be two measures in $\probset(\mathcal{Z})$. Then for any $Z \in \mathcal{A}(\mathcal{Z})$, the following holds, 
\begin{equation*}
  \rho(Z)\leq \left(\exp\left(D_{\beta}(\rho  , \rho' )\right) \rho' (Z)\right)^{\frac{\beta -1}{\beta}}.
\end{equation*}
\end{proposition}
Now thanks to previous works by~\cite{5605338} and~\cite{Vajda1970}, we also get the following results relating the total variation distance and the Renyi divergence.

\begin{proposition}[Inequality between total variation and Renyi divergence]
\label{prop:Inequality-TV-Renyi}
Let $\rho$ and $\rho' $ be two measures in $\probset(\mathcal{Z})$, and $\beta\geq1$. Then the following holds,
$$D_{TV}(\rho  , \rho' ) \leq \min \left(\frac{3}{2}\left(\sqrt{1 + \frac{4 D_{\beta}(\rho  , \rho' )}{9}} - 1\right)^{1/2} ,\  \frac{\exp\left(D_{\beta}(\rho  , \rho' ) +1 \right) -1}{\exp \left(D_{\beta}(\rho  , \rho' ) +1 \right) +1} \right).$$
\end{proposition}

\begin{proof}
Thanks to~\cite{5605338}, one has 
\begin{align*}
    & D_{1}(\rho  , \rho') \geq 2D_{TV}(\rho  , \rho')^{2}+ \frac{4D_{TV}(\rho  , \rho')^{4}}{9}.
    \intertext{From which it follows that}
    & D_{TV}(\rho  , \rho') \leq \frac{3}{2}\left(\sqrt{1 + \frac{4D_{1}(\rho  , \rho')}{9}} - 1\right)^{1/2}.
    \intertext{Moreover, using inequality from~\cite{Vajda1970}, one gets}
    & D_{1}(\rho  , \rho') +1 \geq \log\left(\frac{1 + D_{TV}(\rho  , \rho')}{1 - D_{TV}(\rho  , \rho')} \right).
    \intertext{This inequality leads to the following}
    &\frac{\exp(D_{1}(\rho  , \rho') +1) -1}{\exp(D_{1}(\rho  , \rho') +1) +1} \geq  D_{TV}(\rho  , \rho').
\end{align*}
By combining the above inequalities and by monotony of Renyi divergence regarding $\beta$, one obtains the expected result.
\end{proof}

From now on, we denote $\probmapset_{TV}\left(\alpha,\epsilon\right)$ and  $\probmapset_{\beta}\left(\alpha,\epsilon\right)$ the set of $(\alpha,\epsilon)$-\emph{robust} classifiers respectively for $D_{TV}$ and $D_{\beta}$. The next section gives bounds on the generalization gap in the standard and the adversarial settings for these specific hypothesis classes. 

\section{Risks' gap and Generalization gap for randomized classifiers}
\label{section::LearningAdversarialGap}

As discussed in Section~\ref{sec:AccRobTradeoff}, we can always decompose the adversarial risk of a classifier $\advRisk(\probmap;\alpha_p)$ in two terms. First the standard risk $\Risk(\probmap)$  and second the amount of risk the adversary creates with non-zero perturbations $\advRiskzero(\probmap;\alpha_p)$.
Hence minimizing $\Risk(\probmap)$ can give poor values for $\advRisk(\probmap;\alpha_p)$ and vice-versa. In this section, we upper-bound the risks' gap $\advRiskzero(\probmap;\alpha_p)$, \emph{i.e.} the gap between the risk and the adversarial risk of a robust classifier.

\subsection{Risks' gap for robust classifiers w.r.t. $D_{TV}$}

First, let us consider $\probmap \in \probmapset_{TV}\left(\alpha_p,\epsilon\right)$.
We can control the loss of accuracy under attack of this classifier with the robustness parameter $\epsilon$. 

\begin{theorem}[Risk's gap for robust classifiers w.r.t $D_{TV}$]
\label{th:TVboundRisk}
 Let $\probmap \in \probmapset_{TV}\left(\alpha_p,\epsilon\right)$ . Then we have
\begin{equation*}
    \advRisk(\probmap; \alpha_p)  \leq \Risk(\probmap) + \epsilon \equationspace.
\end{equation*}
\end{theorem}

\begin{proof} Let $\probmap$ be an $(\alpha_p,\epsilon)$-robust classifier \wrt~$D_{TV}$ , $(\inputelement,y ) \sim \groundDistrib$ and $\perturb \in \mathcal{X}$ such that $\norm{\perturb}_p \leq \alpha_p$. By definition of the $\zo$ loss we have
\begin{align*}
&\zerooneloss\left( \probmap(\inputelement + \perturb), y \right) =  \expect_{\hat{y} \sim \probmap(\inputelement+ \perturb)} \left[ \mathds{1}\left\{\hat{y} \neq y\right\} \right]. 
\intertext{Furthermore, by definition of the total variation distance we have}
&\expect_{\hat{y} \sim \probmap(\inputelement + \perturb)} \left[ \mathds{1}\left\{\hat{y} \neq y\right\} \right] - \expect_{\hat{y} \sim \probmap(\inputelement)} \left[ \mathds{1}\left\{\hat{y} \neq y\right\} \right] \leq D_{TV}( \probmap(\inputelement),\probmap(\inputelement+\perturb)). 
\intertext{Since $\probmap \in \probmapset_{TV}\left(\alpha_p,\epsilon\right)$, the above amounts to write}
&\zerooneloss\left( \probmap(\inputelement + \perturb), y \right) - \zerooneloss\left( \probmap(\inputelement), y \right) \leq \epsilon. 
\intertext{Finally, this holds for any $(\inputelement,y) \sim \groundDistrib$ and any $\alpha_p$ bounded perturbation $\perturb$, then we get}
& \expect_{(\inputelement,y) \sim \groundDistrib} \left[ \sup _{ \perturb \in B_p(\alpha_p)} \zerooneloss\left( \probmap(\inputelement + \perturb), y \right) \right] - \expect_{(\inputelement,y) \sim \groundDistrib} \left[ \zerooneloss\left( \probmap(\inputelement), y \right) \right] \leq \epsilon.
\end{align*}
The above inequality concludes the proof.
\end{proof}
This result means that if we can design a class $\probmapset_{TV}\left(\alpha_p,\epsilon\right)$ with small enough $\epsilon$, then minimizing the risk of $\probmap \in \probmapset_{TV}\left(\alpha_p,\epsilon\right)$ is also sufficient to control the adversarial risk. It is relatively easy to obtain, but it has an interesting consequence on the understanding we have of the trade-off between robustness and accuracy. 
It says that there exists some classes of randomized classifiers for which robustness and standard accuracy may not be at odds, since we can upper-bound the maximal loss of accuracy the model may suffer under attack. This questions previous intuitions developed on deterministic classifiers by~\cite{su2018robustness,10.5555/3327546.3327734,tsipras2018robustness} and \cite{zhang2019theoretically} and advocates for the use of randomization schemes as defenses against adversarial attacks. Note, however, that we did not evade the trade-off between robustness and accuracy, we only showed that with certain hypothesis classes it can be controlled.

\subsection{Risks' gap for robust classifiers w.r.t. $D_{\beta}$}

We now extend the previous results the Renyi divergence. We show that, for any randomized classifier in $\probmapset_{\beta}\left(\alpha_p,\epsilon\right)$, we can bound the gap between the risk and the adversarial risk of $\probmap$. Using the Renyi divergence, the factor that controls the classifier’s loss of accuracy under attack can be either multiplicative or additive, and depends both on the robustness parameter $\epsilon$ and on the divergence parameter $\beta$.

\begin{theorem}[Multiplicative risks' gap for Renyi-robust classifiers]
\label{th:multiplicative}
 Let $\probmap \in \probmapset_{\beta}\left(\alpha_p,\epsilon\right)$. Then we have
\begin{equation*}
    \advRisk(\probmap;\alpha_p) \leq \left(e^{\epsilon} \Risk(\probmap)\right)^{\frac{\beta-1}{\beta}}.
\end{equation*}
\end{theorem}

\begin{proof}
Let $\probmap$ be an $(\alpha_p,\epsilon)$-robust classifier \wrt~$D_{\beta}$, $(\inputelement,y ) \sim \groundDistrib$ and $\perturb \in \mathcal{X}$ such that $\norm{\perturb}_p \leq \alpha_p$. With the same reasoning as above, and with Proposition~\ref{prop::renyi}, we get 
\begin{align*}
\zerooneloss \left( \probmap(\inputelement + \perturb), y \right) = ~&  \expect_{\hat{y} \sim \probmap(\inputelement+ \perturb)} \left[ \mathds{1}\left\{\hat{y} \neq y\right\} \right]\\
= ~& \proba_{\hat{y} \sim \probmap(\inputelement+ \perturb)} \left[\hat{y} \neq y\right]\\
\leq ~&\left(e^{ D_{\beta}\left( \probmap(\inputelement +\perturb),\probmap(\inputelement) \right)} \proba_{\hat{y} \sim \probmap(\inputelement)} \left[\hat{y} \neq y \right]\right)^{\frac{\beta-1}{\beta}} \quad \text{(Prop.~\ref{prop::renyi})}\\
= ~&\left(e^{ D_{\beta}\left( \probmap(\inputelement +\perturb),\probmap(\inputelement) \right)} \expect_{\hat{y} \sim \probmap(\inputelement)} \left[ \mathds{1}\left\{\hat{y} \neq y\right\} \right]\right)^{\frac{\beta-1}{\beta}}\\
\leq ~&\left(e^{\epsilon} \zerooneloss\left( \probmap(\inputelement), y \right) \right)^{\frac{\beta-1}{\beta}} \equationspace. \\
\intertext{Since this holds for any $(\inputelement,y) \sim \groundDistrib$ and any $\alpha_p$ bounded perturbation $\perturb$, we get }
  \advRisk(\probmap; \alpha_p) = ~&\expect_{(\inputelement,y)\sim \mathcal{D}}\left[ \sup_{ \perturb \in B_p(\alpha_p)} \zerooneloss \left(\probmap( \inputelement+\perturb), y\right) \right]\\
  \leq ~&\expect_{(\inputelement,y)\sim \mathcal{D}}\left[ e^{\frac{\beta-1}{\beta}\epsilon}   \zerooneloss \left(\probmap( \inputelement), y\right)^{\frac{\beta-1}{\beta}} \right]\\
  \leq ~&e^{\frac{\beta-1}{\beta}\epsilon} \expect_{(\inputelement,y)\sim \mathcal{D}}\left[  \zerooneloss \left(\probmap( \inputelement), y\right)^{\frac{\beta-1}{\beta}}\right] \equationspace.
  \intertext{Finally, using the Jensen inequality, one gets}
  \leq ~& e^{\frac{\beta-1}{\beta}\epsilon} \expect_{(\inputelement,y)\sim \mathcal{D}}\left[  \zerooneloss \left(\probmap( \inputelement), y\right)\right ]^{\frac{\beta-1}{\beta}} =\left(e^{\epsilon} \Risk(\probmap)\right)^{\frac{\beta-1}{\beta}} \equationspace.
 \end{align*}
 The above inequality concludes the proof.
\end{proof}

This first result gives a multiplicative bound on the gap between the standard and adversarial risks. This means that if we can design a class $\probmapset_{\beta}\left(\alpha_p,\epsilon\right)$ with small enough $\epsilon$, and big enough $\beta$, then minimizing the risk of any $\probmap \in \probmapset_{\beta}\left(\alpha_p,\epsilon\right)$ is sufficient to also minimize the adversarial risk of $\probmap$. Nevertheless, multiplicative factors are not easy to analyze. 
\begin{remark}
More general bounds can be computed if we assume that for every randomized classifier $\probmap$ there exists a convex function $\mathbf{f}$ such that for all $\inputelement$ and $\perturb$ with $\lVert\perturb\rVert_p\leq \alpha_p$, we have $\probmap(\inputelement)(Z)\leq \mathbf{f}(\probmap(\inputelement+\perturb)(Z))$ for all measurable sets $Z$. In this case, we get $\advRisk(\probmap;\alpha_p) \leq \mathbf{f}\left( \Risk(\probmap)\right)$. This has a close link with randomized smoothing~\cite{KolterRandomizedSmoothing} and $f$-differential privacy~\cite{dong2019gaussian} where both try to fit the best possible $\mathbf{f}$ using Neyman-Pearson lemma.
\end{remark}
The following result provides an additive counterpart to Theorem~\ref{th:multiplicative}. It gives a control over the loss of accuracy under attack with respect to the robustness parameter $\epsilon$ and the Shannon entropy of $\probmap$.

\begin{theorem}[Additive risks' gap for Renyi-robust classifiers]
\label{th:RenyiboundRisk}
Let $\probmap \in \probmapset_{\beta}\left(\alpha_p,\epsilon\right)$, then we have
$$ \advRisk(\probmap; \alpha_p)-\Risk(\probmap) \leq 1-e^{-\epsilon}  \expect_{\inputelement \sim \mathcal{D}_{\mid \inputspace}}\left[e^{-H(\probmap(\inputelement))}\right]$$
where $H$ is the Shannon entropy (\emph{i.e.} for any $\rho \in \probset\left(\mathcal{Y}\right), H(\rho)= -\sum\limits_{k \in \outputspace} \rho_k \log(\rho_k)$) and $\mathcal{D}_{\mid \inputspace}$ is the marginal distribution of $\groundDistrib$ for $\inputspace$.
\end{theorem}


\begin{proof}
Let $\probmap \in  \probmapset_{\beta}\left(\alpha_p,\epsilon\right)$, then 
\begin{align*}
&\advRisk(\probmap;\alpha_p)-\Risk(\probmap) \\ 
= ~& \mathbb{E}_{(\inputelement,y) \sim \groundDistrib}\left[ \sup_{ \perturb \in B_p(\alpha_p)} \zerooneloss\left( \probmap(\inputelement + \perturb) , y \right) -  \zerooneloss\left( \probmap(\inputelement) , y \right) \right].
\intertext{By definition of the $\zo$ loss, this amounts to write}
= ~&\mathbb{E}_{(\inputelement,y) \sim \groundDistrib}\left[ \sup_{ \perturb \in B_p(\alpha_p)} \mathbb{E}_{\hat{y}_{\text{adv}}\sim \probmap(\inputelement+\perturb), \hat{y} \sim \probmap(\inputelement) }\left[ \mathds{1}\left(\hat{y}_{\text{adv}}\neq y\right)-  \mathds{1}\left(\hat{y}\neq y\right) \right]\right] \\
\leq ~&\mathbb{E}_{(\inputelement,y) \sim \groundDistrib}\left[ \sup_{ \perturb \in B_p(\alpha_p)} \mathbb{E}_{\hat{y}_{\text{adv}}\sim \probmap(\inputelement+\perturb), \hat{y} \sim \probmap(\inputelement)}\left[ \mathds{1}\left(\hat{y}_{\text{adv}}\neq \hat{y}\right)\right]\right]\\
= ~&\mathbb{E}_{(\inputelement,y) \sim \groundDistrib}\left[\sup_{ \perturb \in B_p(\alpha_p)}\mathbb{P}_{\hat{y}_{\text{adv}}\sim \probmap(\inputelement+\perturb),\hat{y}\sim \probmap(\inputelement)} \left [ \hat{y}_{\text{adv}}\neq \hat{y} \right ] \right] \\
= ~&\mathbb{E}_{(\inputelement,y) \sim \groundDistrib}\left[\sup_{ \perturb \in B_p(\alpha_p)} 1 - \mathbb{P}_{\hat{y}_{\text{adv}}\sim \probmap(\inputelement+\perturb),\hat{y}\sim \probmap(\inputelement)} \left [ \hat{y}_{\text{adv}} = \hat{y} \right ] \right] \\
= ~&\mathbb{E}_{(\inputelement,y) \sim \groundDistrib}\left[\sup_{ \perturb \in B_p(\alpha_p)} 1 - \sum_{i=1}^K  \probmap(\inputelement)_i \times \probmap(\inputelement + \perturb)_i \right] \equationspace. 
\end{align*}
Now, note that for any $(\inputelement,y) \sim \groundDistrib$ and $\perturb \in \inputspace$, by definition of a probability vector in $\probset\left( \outputspace \right)$, and thanks to Jensen inequality we can write
\begin{align*}
&\sum_{i=1}^K  \probmap(\inputelement)_i \times \probmap(\inputelement + \perturb)_i \geq \exp\left(\sum_{i=1}^K \probmap(\inputelement)_i \log \probmap(\inputelement + \perturb)_i\right).
\end{align*}
Then by definition of the entropy and the Kullback Leibler divergence we have
\begin{align*}
& \exp\left(\sum_{i=1}^K \probmap(\inputelement)_i \log \probmap(\inputelement + \perturb)_i\right) =\exp\big(-D_{1}\left(\probmap(\inputelement),\probmap(\inputelement + \perturb) \right) - H\left(\probmap(\inputelement) \right) \big).
\end{align*} 
Finally, by combining the above inequalities and since $\probmap \in \probmapset_{\beta}\left(\alpha_p,\epsilon\right)$ we get 
\begin{align*}
&\mathbb{E}_{(\inputelement,y) \sim \groundDistrib}\left[\sup_{ \perturb \in B_p(\alpha_p)}\mathbb{P}_{\hat{y}_{\text{adv}}\sim \probmap(\inputelement+\perturb),\hat{y}\sim \probmap(\inputelement)}(\hat{y}_{\text{adv}}\neq \hat{y})\right]\\
 \leq ~& \mathbb{E}_{(\inputelement,y) \sim \groundDistrib}\left[\sup_{ \perturb \in B_p(\alpha_p)} 1-e^{- D_{1}(\probmap(\inputelement),\probmap(\inputelement+\perturb))-H(\probmap(\inputelement))} \right]\\
\leq ~& \mathbb{E}_{(\inputelement,y) \sim \groundDistrib}\left[1-e^{-\epsilon-H(\probmap(\inputelement))} \right] = 1-e^{-\epsilon}\mathbb{E}_{\inputelement \sim \groundDistrib_{\mid \inputspace}}\left[e^{-H(\probmap(\inputelement))}\right]\equationspace.
\end{align*}
The above inequality concludes the proof.
\end{proof}

This result is interesting because it relates the accuracy of $\probmap$ with the bound we obtain. 
In words, when $\probmap(\inputelement)$ has large entropy (\emph{i.e.} $H(\probmap(\inputelement))\rightarrow \log(K)$) the output distribution tends towards the uniform distribution; hence $\epsilon\rightarrow0$. This means that the classifier is very robust but also completely inaccurate, since it outputs classes uniformly at random.   
On the opposite, if $H(\probmap(\inputelement))\rightarrow 0$, then $\epsilon\rightarrow\infty$. The classifier may be accurate, but it is not robust anymore (at least according to our definition). Hence we need to find a classifier that achieves a trade-off between robustness and accuracy. 

\section{Standard Generalization gap}
\label{section::GeneralizationBoundAdvGap} 

In this section we devise generalization gap bounds for randomized classifiers when they are robust according either to the total variation distance or the Renyi divergence. To do so, we upper-bound the Rademacher complexity of the loss space for TV-robust classifiers $$ \loss_{\probmapset_{TV}\left(\alpha_p,\epsilon\right)}\equaldef \{ (\inputelement,y) \mapsto \zerooneloss(\hypothesis(\inputelement),y)  \mid \probmap \in \probmapset_{TV}\left(\alpha_p,\epsilon\right) \}. $$ The \emph{empirical Rademacher complexity}, first introduced by~\cite{bartlett2002rademacher}, is one of the standard measures of generalization gap. It is particularly useful to obtain quality bounds for complex classes such as neural networks since it does not depend on the number of parameters in the network contrary to combinatorial notions such as the \emph{VC dimension}.

\begin{definition}[Rademacher complexity]
For any class of real-valued functions $\mathcal{F} \equaldef \{(\inputelement,y)\mapsto \R \}$, given a training sample $\fullSample=\{(\vectorsym{x_1},y_1), \dots ,(\vectorsym{x_n},y_n)\}$, the \emph{empirical Rademacher complexity} of $\mathcal{F}$ is defined as
\begin{equation*}
    \Rademacher_{\mathcal{S}}(\mathcal{F})\equaldef \frac1n \expect_{r_i}\left[ \sup_{f \in \mathcal{F}} \sum_{i=1}^{n} r_i f(\vectorsym{x_i},y_i) \right] \equationspace,
\end{equation*}
where $r_i$ are \emph{i.i.d.} drawn from a Rademacher measure (\emph{i.e.} $\proba(r_i = 1) = \proba(r_i = -1) = \frac12$).
\end{definition}


The empirical Rademacher complexity measures the uniform convergence rate of the empirical risk towards the risk on the function class $\mathcal{F}$ as demonstrated by \cite{mohri2018foundations}. Thanks to this notion of complexity, we can bound with high probability the generalization gap of any hypothesis $\probmap$ in a class $\mathcal{M}$.

\begin{theorem}[\cite{mohri2018foundations}]
\label{thm:RademacherandGenClassicalTheorem}
Let $\probmapset$ be a class of possibly randomized classifiers and $\loss_{\probmapset} \equaldef \{ \loss_{\probmap} :(\inputelement,y) \mapsto \zerooneloss\left(\probmap(\vectorsym{x}),y\right) \mid \probmap \in \probmapset \}$. Then for any $\delta \in (0,1)$, with probability at least $1-\delta$, the following holds for any $\probmap \in \probmapset_{TV}\left(\alpha_p,\epsilon\right)$,
\begin{equation*}
    \Risk\left( \probmap \right) - \Riskemp\left( \probmap \right) \leq 2 \Rademacher_{\mathcal{S}}(\loss_{\probmapset}) + 3 \sqrt{\cfrac{\ln(2/\delta)}{2n}} \equationspace.
\end{equation*}
\end{theorem}

\subsection{Generalization error for robust classifiers}

Accordingly, we want to upper bound the empirical Rademacher complexity of  $\loss_{\probmapset_{TV}\left(\alpha_p,\epsilon\right)}$, which motivates the following definition. 
\begin{definition}[$\alpha$-covering and external covering number]
Let us consider $(\inputspace , \norm{.}_p)$ a vector space equipped with the $\lp$ norm, $B \subset \inputspace$ and $\alpha \geq 0$. Then 
\begin{itemize}
    \item $C =\{ \vectorsym{c_1}, \dots, \vectorsym{c_m} \}$ is an $\alpha$-covering of $B$ for the $\lp$ norm if for any $\inputelement \in B$ there exists $\vectorsym{c_i} \in C$ such that $\norm{\inputelement - \vectorsym{c_i}}_p \leq \alpha$.
    \item The external covering number of $B$ writes $N\left(B,\norm{.}_p,\alpha\right)$. It is the minimal number of points one needs to build an $\alpha$-covering of $B$ for the $\lp$ norm.
\end{itemize}
\end{definition}
The covering number is a well-known measure that is often used in statistical learning theory~\cite{shalev2014understanding} and asymptotic statistics~\cite{van2000asymptotic} to evaluate the complexity of a set of functions.
Here we use it to evaluate the number of $\lp$ balls we need to cover the training samples, which gives us the following bound on the Rademacher complexity of $\loss_{\probmapset_{TV}\left(\alpha_p,\epsilon\right)}$.

\begin{theorem}[Rademacher complexity for TV-robust classifiers]
\label{thm:rad_tv}
Let $\loss_{\probmapset_{TV}\left(\alpha_p,\epsilon\right)}$ be the loss function class associated with $\probmapset_{TV}\left(\alpha_p,\epsilon\right)$. Then, for any $\mathcal{S}:=\{(\vectorsym{x_1},y_1), \dots  , (\vectorsym{x_n},y_n)\}$, the following holds,
\begin{equation*}
    \mathfrak{R}_{\mathcal{S}}\left(\loss_{\probmapset_{TV}\left(\alpha_p,\epsilon\right)}\right ) \leq \sqrt{\cfrac{ N \times K }{n}}+\epsilon.
\end{equation*}
 Where $N =N\left( \{\vectorsym{x_1},\dots , \vectorsym{x_n}\}, \norm{.}_p, \alpha_p \right)$ is the $\alpha_p$-external covering number of the inputs $\{\vectorsym{x_1},\dots , \vectorsym{x_n}\}$ for the $\lp$ norm.
\end{theorem}

\begin{proof}
Let us denote $\mathcal{S}:=\{(\vectorsym{x_1},y_1), \dots  , (\vectorsym{x_n},y_n)\}$ and $N=N\left( \{\vectorsym{x_1},\dots , \vectorsym{x_n}\}, \norm{.}_p, \alpha_p \right)$. By definition of a covering number, there exists $C= \{\vectorsym{c_1} , \dots, \vectorsym{c_N}\}$ an $\alpha_p$-covering of $\{\vectorsym{x_1},\dots \vectorsym{x_n}\}$ for the $\lp$ norm. Furthermore, for $j\in\{1,\dots ,N\}$ and $y \in\{1,\dots ,K\}$, we define $$E_{y,j} = \left\{ i \in \{1,\dots, n\} 
~\st~ y_i = y \text{ and } \argmin\limits_{l \in \{ 1, \dots , N\}} \norm{x_i - c_l} = j\right\}.$$ 
We also denote $E_j = \underset{y \in [K]}{\cup} E_{y,j}$. Finally, we denote $\loss_{\probmap} :(\inputelement,y) \mapsto \zerooneloss\left(\probmap(\vectorsym{x}),y\right)$. Then, by definition of the empirical Rademacher complexity, we can write 
\begin{align*}
    \mathfrak{R}_{\mathcal{S}}\left(\loss_{\probmapset_{TV}\left(\alpha_p,\epsilon\right)}\right ) = ~& \frac1n \expect_{r_i}\left[ \sup_{\probmap \in \probmapset_{TV}\left(\alpha_p,\epsilon\right)}
     \sum_{i=1}^{n} r_i \loss_{\probmap}(\vectorsym{x_i}, y_i)\right].
     \intertext{Then we can use $E_j$ to write}
     \mathfrak{R}_{\mathcal{S}}\left(\loss_{\probmapset_{TV}\left(\alpha_p,\epsilon\right)}\right )  = ~&\frac1n \expect_{r_i}\left[ \sup_{\probmap \in \probmapset_{TV}\left(\alpha_p,\epsilon\right)} \sum_{j=1}^N\sum_{i\in E_{j}} r_i \loss_{\probmap}(\vectorsym{x_i}, y_i) \right].
\intertext{Furthermore for any $ \probmap \in \probmapset_{TV}\left(\alpha_p,\epsilon\right)$ and $i\in E_j$, there exists $\epsilon_i \in[-\epsilon,\epsilon]$ such that: $\loss_{\probmap}(\vectorsym{x_i}, y_i) = \loss_{\probmap}(\vectorsym{c_j},y_i)+\epsilon_i$. Then we have}
     \mathfrak{R}_{\mathcal{S}}\left( \loss_{\probmapset_{TV}\left(\alpha_p,\epsilon\right)} \right) \leq ~&\frac1n \expect_{r_i}\left[ \sup_{\probmap \in \probmapset_{TV}\left(\alpha_p,\epsilon\right)} \sum_{j=1}^N\sum_{i\in E_{j}} r_i \loss_{\probmap}(\vectorsym{c_j},y_i) \right]\\
     + ~& \frac1n \expect_{r_i}\left[ \sup_{\epsilon_i\in[-\epsilon,\epsilon]} \sum_{j=1}^N\sum_{i\in E_{j}} r_i \epsilon_i \right].
\end{align*}
Let us start by studying the second term. We have 
\begin{align*}
     \frac1n \expect_{r_i}\left[ \sup_{\epsilon_i\in[-\epsilon,\epsilon]} \sum_{j=1}^N\sum_{i\in E_{j}} r_i \epsilon_i \right] =\frac1n \expect_{r_i}\left[ \sup_{\epsilon_i\in[-\epsilon,\epsilon]} \sum_{i=1}^n r_i \epsilon_i \right] = \frac1n \sum_{i=1}^n \epsilon =\epsilon. 
\end{align*}
Now looking at the first term. Since $\loss_{\probmap}(\inputelement,y)\in[0,1]$ for all $(\inputelement,y)$ we have
\begin{align*}
      \frac1n \expect_{r_i}\left[ \sup_{\probmap \in \probmapset_{TV}\left(\alpha_p,\epsilon\right)} \sum_{j=1}^N\sum_{i\in E_{j}} r_i \loss_{\probmap}(\vectorsym{c_j},y_i) \right]
      = ~& \frac1n \expect_{r_i}\left[ \sup_{\probmap \in \probmapset_{TV}\left(\alpha_p,\epsilon\right)} \sum_{j=1}^N\sum_{y=1}^K  \loss_{\probmap}(\vectorsym{c_j},y) \sum_{i\in E_{y,j}}r_i \right]\\
      \leq ~&\frac1n \expect_{r_i}\left[  \sum_{j=1}^N\sum_{y=1}^K  \abs{ \sum_{i\in E_{y,j}}r_i}\right] \equationspace.
 \end{align*}
 Finally using the Khintchine inequality and the Cauchy Schartz inequality we get 
 \begin{align*}
      \frac1n \expect_{r_i}\left[  \sum_{j=1}^N\sum_{y=1}^K  \abs{ \sum_{i\in E_{y,j}}r_i}\right]\leq ~&\frac1n \sum_{j=1}^N\sum_{y=1}^K \sqrt{\abs{E_{y,j}}} \quad \text{(Khintchine)}\\
      \leq ~&\frac1n \sqrt{N\times K}\sqrt{\sum_{j=1}^N\sum_{y=1}^K \abs{E_{y,j}}} \quad \text{(Cauchy)} \\
      = ~&\sqrt{\frac{N\times K}{n}}.
 \end{align*}
By combining the upper-bounds we have for each term, we get the expected result,
 \begin{align*}
     \mathfrak{R}_{\mathcal{S}}\left(\loss_{\probmapset_{TV}\left(\alpha_p,\epsilon\right)}\right ) \leq \sqrt{\frac{N\times K}{n}}+\epsilon.
 \end{align*}
\end{proof}

The above result means that, if we can cover the $n$ training samples with $O(1)$ balls, then we can bound the generalization gap of any randomized classifier $\probmap \in \probmapset_{TV}\left(\alpha_p,\epsilon\right)$ by $O\left(\frac{1}{\sqrt{n}}\right) + \epsilon$. Furthermore, a natural corollary of Theorem~\ref{thm:rad_tv} bounds the Rademacher complexity of the class $\loss_{\probmapset_{\beta}\left(\alpha_p,\epsilon\right)}$.

\begin{corollary}
\label{cor:rad_rob}
Let $\loss_{\probmapset_{\beta}\left(\alpha_p,\epsilon\right)}$ be the loss function class associated with $\probmapset_{\beta}\left(\alpha_p,\epsilon\right)$. Then, for any $\mathcal{S}:=\{(\vectorsym{x_1},y_1), \dots  , (\vectorsym{x_n},y_n)\}$, the following holds,
\begin{equation*}
    \mathfrak{R}_{\mathcal{S}}\left(\loss_{\probmapset_{\beta}\left(\alpha_p,\epsilon\right)}\right ) \leq \sqrt{\cfrac{ N \times K }{n}}+ \min \left(\cfrac{3}{2}\left(\sqrt{1 + \cfrac{4\epsilon}{9}} - 1\right)^{1/2}, \cfrac{e^{\epsilon +1} -1}{e^{\epsilon +1} +1}\right).
\end{equation*}
 Where $N =N\left( \{\vectorsym{x_1},\dots , \vectorsym{x_n}\}, \norm{.}_p, \alpha_p \right)$ is the $\alpha_p$-external covering number of the inputs $\{\vectorsym{x_1},\dots , \vectorsym{x_n}\}$ for the $\lp$ norm.
\end{corollary}

\begin{proof}
This corollary is an immediate consequence of Theorem~\ref{thm:rad_tv} and Proposition~\ref{prop:Inequality-TV-Renyi}.
\end{proof}
 Thanks to Theorems~\ref{thm:RademacherandGenClassicalTheorem} and~\ref{thm:rad_tv} and Corollary~\ref{cor:rad_rob}, one can easily bound the generalization gap of robust randomized classifiers.

\subsection{Discussion and dimensionality issues}

\cite{xu2012robustness} previously studied generalization bounds for learning algorithms based on their robustness. 
Although we use very different proof techniques, their results and ours are similar. More precisely, both analyses conclude that robust models generalize well if the training samples have a small covering number. Note, however, that we base our formulation on an \emph{adaptive partition} of the samples, while the initial paper from~\cite{xu2012robustness} only focuses on a fixed partition of the input space. The interested reader can refer to the discussion section in~\cite{xu2012robustness} for more details. 

These findings seem to contradict the current line of works on the hardness of generalization in the adversarial setting. In fact, if the ground truth distribution is sufficiently concentrated (\emph{e.g.} lies in a low dimensional subspace of $\inputelement$), a small number of balls can cover $\fullSample$ with high probability; hence $N = O(1)$. This means that we can learn robust classifiers with the same sample complexity as in the standard setting. But if the ground truth distribution is not concentrated enough, the training samples will be far one from another; hence forcing the covering number to be large. In the worse case scenario, we need to cover the whole space $[0,1]^d$ giving a covering number $N = O\left(\frac{1}{(\alpha_p)^d }\right)$ which is exponential in the dimension of the problem.

Therefore, in the worst-case scenario, our bound is in $O\left(\frac{1}{(\alpha_p)^d \sqrt{n}}\right) + \epsilon$. When $\alpha_p$ is small and the dimension of the problem is high, this bound is too large to give any meaningful insight on the generalization gap of the problem.
Therefore, we still need to tighten our analysis to show that robust learning for randomized classifiers is possible in high dimensional spaces. 

\begin{remark}
Note that, we provided a very general result for randomized classifiers under the only assumption that they are robust \wrt~the total variation distance. Our result applies to any class of classifiers and not only linear classifiers or  one-hidden layer neural networks. To build a finer analysis, and to evade the curse of dimensionality, we should consider designing specific sub-classes $\probmapset \subset \probmapset_{TV}\left(\alpha_p,\epsilon\right)$ and adapt the proofs to make the term $N$ smaller in the worst-case scenario.  
\end{remark}


\section{Building robust randomized classifiers} 
\label{section::Noisescheme}
In this section we present a simple yet efficient way to transform a non-robust, non-randomized classifier into a robust randomized classifier. To do so, we use a key property of both the Renyi divergence and the total variation distance called the \textit{Data processing inequality}. It is a well-known result from information theory which states that \textit{``post-processing cannot increase information''}. The data processing inequality is as follows. 
\begin{theorem}[\cite{cover2012elements}]
\label{th::Dataprocessing}
Let us consider two arbitrary spaces $\mathcal{Z}, \mathcal{Z}'$, $\rho,\rho' \in \probset\left( \mathcal{Z} \right)$ and $D \in \{D_{TV},D_{\beta}\}$. Then for any $\psi : \mathcal{Z} \rightarrow \mathcal{Z}'$ we have 
$$  D\left( \psi \#\rho, \psi \#\rho' \right) \leq D\left( \rho,\rho' \right),$$
where $\psi \#\rho$ denotes the pushforward of distiburtion $\rho$ by $\psi$.
\end{theorem}

In the context of robustness to adversarial examples, we use the data processing inequality to ease the design of robust randomized classifiers. In particular, let us suppose that we can build a randomized pre-processing $\mathfrak{p}: \inputspace \rightarrow \probset\left( \inputspace \right)$ such that for any $\inputelement \in \mathcal{X}$ and any $\alpha_p$-bounded perturbation $\perturb$, we have 
\begin{equation}
   D\left(\mathfrak{p}(\inputelement), \mathfrak{p}(\inputelement + \perturb) \right) \leq \epsilon,
\text{ with }D \in  \{D_{TV}, D_\beta \}. 
\end{equation}
Then, thanks to the data processing inequality, we can take any deterministic classifier $\hypothesis$ to build an $(\alpha_p,\epsilon)$ robust classifier w.r.t $D$ defined as $\probmap : \inputelement \mapsto  \hypothesis \# \mathfrak{p}(\inputelement)$. This considerably simplifies the problem of building a class of robust models. Therefore, we want to build $\mathfrak{p}$ a randomized pre-processing for which we can control the Renyi divergence and/or total variation distance between two inputs. To do this, we analyze the simple procedure of injecting random noise directly on the image before sending it to a classifier. Since the Renyi divergence and the total variation distances are particularly well suited to the study of Gaussian distributions, we first use this type of noise injection. More precisely, in this section, we focus on a mapping that writes as follows.
\begin{equation}
    \mathfrak{p}: \inputelement \mapsto \mathcal{N}\left(\inputelement, \Sigma \right),
\end{equation}
for some given non-degenerate covariance matrix $\Sigma \in \mathcal{M}_{d\times d}(\R)$.
We refer the interested reader to~\cite{pinot2019theoretical} for more general classes of noise, namely exponential families. 
Let us now evaluate the maximal variation of Gaussian pre-processing $\mathfrak{p}$ when applied to an image $\inputelement\in \inputspace$ with and without perturbation. 

\begin{lemma}
\label{gaussRenyi} Let $\beta>1$, $\inputelement, \perturb \in \inputspace$ and $\Sigma \in \mathcal{M}_{d \times d}(\R)$ a non-degenerate covariance matrix. Let $\rho = \mathcal{N}(\inputelement,\Sigma)$ and $\rho'=\mathcal{N}(\inputelement + \perturb,\Sigma)$, then $D_{\beta}(\rho,\rho') = \frac{ \beta }{2} \norm{\perturb}_{\Sigma^{- 1}}^2 $.
\end{lemma}

\begin{proof}
Let $\beta>1$. Let us denote $g$ and $g'$ respectively the probability density functions of $\rho$ and $\rho'$ with respect to the Lebesgue measure. We also set $\inputelement' = \inputelement + \perturb$ for readability. Then we have
\begin{align*}
D_\beta(\rho,\rho') &=\frac{1}{\beta-1}\log \expect_{\vectorsym{z} \sim \rho'} \left[ \left(\frac{g(\vectorsym{z}) }{g'(\vectorsym{z})}\right)^\beta  \right ]\\
 = & \frac{1}{\beta-1} \log \expect_{\vectorsym{z} \sim \rho'} \Big[ \exp \Big(\frac{\beta}{2}\big((\vectorsym{z}-\inputelement')^\intercal \Sigma^{-1}(\vectorsym{z}-\inputelement') - (\vectorsym{z}-\inputelement)^\intercal \Sigma^{-1}(\vectorsym{z}-\inputelement) \big) \Big) \Big]. \intertext{By change of variable we get}
 = & \frac{1}{\beta-1}\log \expect_{\vectorsym{z} \sim \mathcal{N}(0,\Sigma) }\left[ \exp\left(\frac{\beta}{2}\big(\vectorsym{z}^\intercal\Sigma^{-1}\vectorsym{z}-(\vectorsym{z}+\perturb)^\intercal \Sigma^{-1}(\vectorsym{z}+\perturb) \big) \right) \right]\\
 = &  \frac{1}{\beta-1} \log \expect_{\vectorsym{z} \sim \mathcal{N}(0,\Sigma) }\left[ \exp\left(\frac{\beta}{2}\left(- 2\vectorsym{z}^\intercal\Sigma^{-1}\perturb- \norm{\perturb}_{\Sigma^{-1}}^2\right)\right) \right] \\
 = &  \frac{1}{\beta-1} \log \int_{\mathbb{R}^d} \frac{\exp\left(-\frac{1}{2}\vectorsym{z}^\intercal\Sigma^{-1}\vectorsym{z} - \frac{\beta}{2}2\vectorsym{z}^\intercal\Sigma^{-1}\perturb -  \frac{\beta}{2}\norm{\perturb}_{\Sigma^{-1}}^2\right)}{(2 \pi)^d \det(\Sigma)^{d/2}} d\vectorsym{z} \equationspace.
 \end{align*}
 Furthermore, for any $\vectorsym{z} \in \mathbb{R}^d$, we have 
 \begin{align*}
 & -\frac{1}{2}\vectorsym{z}^\intercal\Sigma^{-1}\vectorsym{z} - \frac{\beta}{2}2\vectorsym{z}^\intercal\Sigma^{-1}\perturb -  \frac{\beta}{2}\norm{\perturb}_{\Sigma^{-1}}^2 \\
 =& - \frac{1}{2}(\vectorsym{z} + \beta\perturb)^\intercal\Sigma^{-1}(\vectorsym{z} + \beta\perturb) + \frac{\beta^2 - \beta}{2} \norm{\perturb}_{\Sigma^{-1}}^2 \equationspace.
 \end{align*}
 Then we can re-write the Renyi divergence as follows
 \begin{align*}
 D_\beta(\rho,\rho') & = \frac{1}{\beta-1}\log \mathbb{E}_{\vectorsym{z} \sim \mathcal{N}(- \beta \perturb,\Sigma)} \left[\exp\left( \frac{\beta^2 - \beta}{2} \norm{\perturb}_{\Sigma^{-1}}^2 \right)\right]\\
 & = \frac{1}{\beta-1}\log\left(\exp\left( \frac{\beta^2 - \beta}{2} \norm{\perturb}_{\Sigma^{-1}}^2 \right)\right)\\
 & = \frac{ \beta }{2} \norm{\perturb}_{\Sigma^{- 1}}^2 \equationspace.
\end{align*}

This concludes the proof.
\end{proof}

Thanks to the above lemma, we know how to evaluate the level of Renyi-robustness that a Gaussian noise pre-processing brings to a classifier. Now that we have this result, thanks to Proposition~\ref{prop:Inequality-TV-Renyi}, we can also upper-bound the total variation distance between $\mathcal{N}(\inputelement,\Sigma)$ and $\mathcal{N}(\inputelement + \perturb,\Sigma)$. But this bound is not always tight. Besides, we can directly evaluate the total variation distance between two Gaussian distributions as follows.

\begin{lemma}
\label{gaussTV}Let $\inputelement, \inputelement' \in \inputspace$ and $\Sigma \in \mathcal{M}_{d \times d}(\R)$ a non-degenerate covariance matrix. Let $\rho = \mathcal{N}(\inputelement,\Sigma)$ and $\rho'=\mathcal{N}( \inputelement+ \perturb,\Sigma)$, then $D_{TV}(\rho,\rho') = 2\Phi(\frac{\norm{\perturb}_{\Sigma^{-1}}}{2})-1$ with $\Phi$ the cumulative density function of the standard Gaussian distribution.
\end{lemma}

\begin{proof}
 Let us denote $g$ and $g'$ respectively the probability density functions of $\rho$ and $\rho'$ with respect to the Lebesgue measure. Furthermore, we denote $\inputelement' = \inputelement + \perturb$. Then by definition of the total variation distance, we have  $D_{TV}(\rho,\rho)=\rho(Z)-\rho'(Z)$ with $Z=\{\vectorsym{z} ~\st~ g(\vectorsym{z})\geq g'(\vectorsym{z})\}$. In our case $g(\vectorsym{z})\geq g'(\vectorsym{z})$ is equivalent to $$(\vectorsym{z}-\inputelement')^\intercal\Sigma^{-1}(\vectorsym{z}-\inputelement')-(\vectorsym{z}-\inputelement)^\intercal\Sigma^{-1}(\vectorsym{z}-\inputelement)\geq 0.$$
Then with the same simplification as above, we have
\begin{align*}
\rho(Z)& = \proba_{ \vectorsym{z}\sim\mathcal{N}(\inputelement,\Sigma)}\left(( \vectorsym{z}-\inputelement')^\intercal\Sigma^{-1}( \vectorsym{z}-\inputelement')-( \vectorsym{z}-\inputelement)^\intercal\Sigma^{-1}( \vectorsym{z}-\inputelement)\geq 0 \right)\\
& =  \proba_{ \vectorsym{z}\sim\mathcal{N}(0,\Sigma)}\left(( \vectorsym{z}-\perturb)^\intercal\Sigma^{-1}( \vectorsym{z}-\perturb)- \vectorsym{z}^\intercal\Sigma^{-1} \vectorsym{z}\geq 0 \right)\\
& = \proba_{ \vectorsym{z}\sim\mathcal{N}(0,\Sigma)}\left(-2 \vectorsym{z}^\intercal\Sigma^{-1}\perturb+\lVert \perturb\rVert _{\Sigma^{-1}}^2\geq 0 \right)\\
&=\proba_{ \vectorsym{z}\sim\mathcal{N}(0,I_d)}\left( \vectorsym{z}^\intercal\Sigma^{-1/2}\perturb\leq\frac12 \lVert \perturb\rVert _{\Sigma^{-1}}^2\right). 
\intertext{
Furthermore, if $ \vectorsym{z}\sim \mathcal{N}(0,I_d)$ then $ \vectorsym{z}^\intercal\Sigma^{-1/2}\perturb\sim \mathcal{N}(0,\lVert \perturb\rVert _{\Sigma^{-1}}^2)$; hence we also have $\frac{\vectorsym{z}^\intercal\Sigma^{-1/2}\perturb}{\lVert \perturb\rVert _{\Sigma^{-1}} }\sim \mathcal{N}(0,1)$. Accordingly we get }
\rho(Z) &= \proba_{\vectorsym{z}\sim\mathcal{N}(0,1)}\left( \vectorsym{z}\leq\frac12 \lVert \perturb\rVert _{\Sigma^{-1}} \right) = \Phi \left(\frac12 \lVert \perturb\rVert _{\Sigma^{-1}} \right).
\end{align*}
By symmetry we get that $\rho'(A)= 1-\rho(A) = 1-\Phi\left(\frac12 \lVert \perturb\rVert _{\Sigma^{-1}}\right)$. We then get
$$D_{TV}(\mu,\nu) = 2\Phi\left(\frac{\lVert \perturb\rVert _{\Sigma^{-1}}}{2}\right)-1$$ 
which concludes the proof.
\end{proof}

Note that both bounds increase with the Mahalanobis norm of $\perturb$. Furthermore, we see that the greater the entropy of the Gaussian noise we inject, the smaller the distance between distributions. If we simplify the covariance matrix by setting $\Sigma= \sigma^2 I_d$, it means that we can build more or less robust randomized classifiers against $\ell_2$ adversaries, depending on $\sigma$.

\begin{theorem}[Robustness of Gaussian pre-processing]
\label{theorem:noiseinjection}
Let us consider $c: \mathcal{X} \rightarrow \mathcal{Y}$ a deterministic classifier, $\sigma > 0$ and $\mathfrak{p}: \inputelement \mapsto \mathcal{N}(\inputelement, \sigma^2 I_d)$ a pre-processing probabilistic mapping. Then the randomized classifier $\probmap \equaldef c \# \mathfrak{p}$ is 
\begin{itemize}
\item $(\alpha_2, \frac{(\alpha_2)^2 \beta}{2 \sigma})$-robust \wrt~$D_\beta$ against $\ell_2$ adversaries.
\item $(\alpha_2,\ 2 \Phi\left( \frac{\alpha_2}{2 \sigma} \right) - 1)$-robust \wrt~$D_{TV}$ against $\ell_2$ adversaries.
\end{itemize}
\end{theorem}

\begin{proof} Let $\inputelement, \perturb \in \inputspace$ such that $\norm{\perturb}_2 \leq \alpha_2$. Thanks to Lemma~\ref{gaussRenyi} we have
\begin{align*}
D_\beta(\mathfrak{p}(\inputelement),\mathfrak{p}(\inputelement + \perturb)) &=\frac{\beta}{2}\lVert \perturb\rVert_{\Sigma^{-1}}^2 = \frac{\beta}{2 \sigma^2}\lVert \perturb\rVert_{2}^2 \leq \frac{\beta (\alpha_2)^2}{2 \sigma^2}.
\intertext{Similarly, thanks to Lemma~\ref{gaussTV}, we get} 
D_{TV}(\mathfrak{p}(\inputelement),\mathfrak{p}(\inputelement + \perturb)) &= 2\Phi\left(\frac{\lVert \perturb \rVert _{\Sigma^{-1}}}{2} \right)-1 \leq 2\Phi\left(\frac{\alpha_2}{2 \sigma} \right)-1.
\end{align*}
Finally, from the data processing inequality, \ie~ Theorem~\ref{th::Dataprocessing}, we get both \begin{align*}
    D_{\beta}(\probmap(\inputelement),\probmap(\inputelement + \perturb)) &\leq  \frac{\beta (\alpha_2)^2}{2 \sigma^2}, 
    \intertext{and }
    D_{TV}(\probmap(\inputelement),\probmap(\inputelement + \perturb)) &\leq  2\Phi\left(\frac{\alpha_2}{2 \sigma} \right)-1.
\end{align*}
The above inequalities conclude the proof.  
\end{proof}

Theorem~\ref{theorem:noiseinjection} means that we can build simple noise injection schemes as pre-processing of state-of-the-art image classification models and keep track of the maximal loss of accuracy under attack of the resulting randomized classifier. These results also highlight the profound link between randomized classifiers and randomized smoothing as presented by \cite{KolterRandomizedSmoothing}. Even though our findings are of different nature, both techniques use the same base mechanism (Gaussian noise injection). 
Therefore, Gaussian pre-processing is a principled defense method that can be analyzed through several standpoints, including certified robustness and statistical learning theory.

\section{Discussion: Mode preservation and Randomized Smoothing}
\label{sec:modepreservationendRS}

Even though randomized classifiers have some interesting properties regarding generalization error, we can also study them through the prism of deterministic robustness. Let us for example consider the classifier that outputs the class with the highest probability for $\probmap(\inputelement)$, \aka~the mode of $\probmap(\inputelement)$. It writes
\begin{equation}
\label{eq:modeandRandomizedSmoothing}
    \hypothesis_{\text{rob}}: \inputelement \mapsto  \argmax\limits_{k \in [K]} \probmap(\inputelement)_k
\end{equation}

Then checking whether $\hypothesis_{\text{rob}}$ is robust boils down to demonstrating that the mode of $\probmap(\inputelement)$ does not change under perturbation. It turns out that $D_{TV}$ robust classifiers have this property. We call it the mode preservation property of $\probmapset_{TV}(\alpha_p,\epsilon)$.

\begin{proposition}[Mode preservation for $D_{TV}$-robust classifiers] 
\label{prop:modepreservationforTV}
Let $\probmap \in \probmapset_{TV}\left(\alpha_p,\epsilon\right)$ be a robust randomized classifier and $\inputelement \in \mathcal{X}$ such that $\probmap(\inputelement)_{(1)} \geq \probmap(\inputelement)_{(2)} +2 \epsilon$. Then, for any $\perturb \in \mathcal{X}$, the following holds,
\begin{equation*}
\norm{\perturb}_p \leq \alpha_p \implies \hypothesis_{\text{rob}}(\inputelement)  = \hypothesis_{\text{rob}}(\inputelement + \perturb )\enspace.
\end{equation*}
\end{proposition}

\begin{proof}
Let $\inputelement,\perturb  \in \inputspace$ such that $\norm{\perturb}_p \leq \alpha_p$ and $\probmap \in \probmapset_{TV}\left(\alpha_p,\epsilon\right)$
such that $$\probmap(\inputelement)_{(1)} \geq \probmap(\inputelement)_{(2)} +2\epsilon.$$ By definition of $\probmapset_{TV}\left(\alpha_p,\epsilon\right)$, we have that $$D_{TV}(\probmap(\inputelement),\probmap(\inputelement+\perturb))\leq\epsilon.$$ Then, for all $k \in \{1, \dots, K\}$ we have $$\probmap(\inputelement)_{k}-\epsilon\leq\probmap(\inputelement+\perturb)_{k}\leq\probmap(\inputelement)_{k}+\epsilon \equationspace.$$ 
Let us denote $k^*$ the index of the biggest value in $\probmap(\inputelement)$, \ie~$\probmap(\inputelement)_{k^*} =\probmap(\inputelement)_{(1)}$. For any $k\in \{1, \dots, K\}$ with $k \neq k^*$, we have $\probmap(\inputelement)_{k^*} \geq \probmap(\inputelement)_{k} + 2\epsilon$. Finally, for any $k \neq k^*$, we get $$\probmap(\inputelement+\perturb)_{k^*}\geq \probmap(\inputelement)_{k^*}-\epsilon\geq \probmap(\inputelement)_{k}+\epsilon\geq\probmap(\inputelement+\perturb)_{k}.$$
Then, $\argmax\limits_{k \in [K]}\probmap(\inputelement)_{k}=\argmax\limits_{k \in [K]}\probmap(\inputelement+\perturb)_{k}$. This concludes the proof.
\end{proof}
Similarly, we can demonstrate a mode preservation property for robust classifiers w.r.t. the Renyi divergence. 

\begin{proposition}[Mode preservation for Renyi-robust classifiers] Let $\probmap \in \probmapset_{\beta}\left(\alpha_p,\epsilon\right)$ be a robust randomized classifier and $\inputelement \in \mathcal{X}$ such that  $\left(\probmap(\inputelement)_{(1) }\right)^{\frac{\beta}{\beta - 1}} \geq \exp\left( (2-\frac{1}{\beta}) \epsilon \right) \left(\probmap(\inputelement)_{(2)}\right)^{\frac{\beta-1}{\beta}}$. Then, for any $\perturb \in \mathcal{X}$, the following holds,
\begin{equation*}
\norm{\perturb}_p \leq \alpha_p \implies \hypothesis_{\text{rob}}(\inputelement) = \hypothesis_{\text{rob}}(\inputelement + \perturb), 
\end{equation*}
where $\hypothesis_{\text{rob}}(\inputelement) \equaldef \argmax\limits_{k \in [K]}\probmap(\inputelement)_{k}$.
\end{proposition}

\begin{proof}
Let $\inputelement, \perturb \in \inputspace$ such that  $\norm{\perturb}_p \leq \alpha_p$ and $\probmap \in \probmapset_{\beta}\left(\alpha_p,\epsilon \right)$ such that $$\left(\probmap(\inputelement)_{(1)}\right)^{\frac{\beta}{\beta - 1}} \geq \exp\left( (2-\frac{1}{\beta})\epsilon \right) \left(\probmap(\inputelement)_{(2)}\right)^{\frac{\beta-1}{\beta}}.$$
Then by definition of $ \probmapset_{\beta}\left(\alpha_p,\epsilon \right)$, we have $$D_{\beta}(\probmap(\inputelement),\probmap( \inputelement+\perturb)) \leq \epsilon.$$ Furthermore, by using Proposition~\ref{prop::renyi}, for any $k \in \{1 ,\dots, K \}$ we have  $$ (*) \probmap(\inputelement)_{k}\leq\left(\exp(\epsilon)\probmap(\inputelement+\perturb)_k\right)^{\frac{\beta-1}{\beta}}\text{ and } (**) \probmap(\inputelement+\perturb)_{k}\leq\left(\exp(\epsilon)\probmap(\inputelement)_k\right)^{\frac{\beta-1}{\beta}} \equationspace.$$ 
Let us denote $k^*$ the index such that $\probmap(\inputelement)_{k^*} =\probmap(\inputelement)_{(1)} $. Then using $(*)$ we get $$\probmap(\inputelement+\perturb)_{k^*} \geq \exp(-\epsilon)(\probmap(\inputelement)_{k^*})^{\frac{\beta}{\beta-1}}.$$
Furthermore for any $k \in \{1, \dots ,K\}$ where $k \neq k^*$, we can use the assumption we made on $\probmap$ to get $$\exp(-\epsilon)(\probmap(\inputelement)_{k^*})^{\frac{\beta}{\beta-1}}\geq\exp(\frac{\beta-1}{\beta}\epsilon)(\probmap(\inputelement)_k)^{\frac{\beta-1}{\beta}}.$$
Finally, using $(**)$ we have
$$\exp(\frac{\beta-1}{\beta}\epsilon)(\probmap(\inputelement)_k)^{\frac{\beta-1}{\beta}} \geq\probmap(\inputelement + \perturb)_{k}.$$
The above gives us $\argmax\limits_{k \in [K] }\probmap(\inputelement)_{k}=\argmax\limits_{k \in [K] }\probmap(\inputelement+\perturb)_{k}$. This concludes the proof.
\end{proof}

Coming back to the decomposition in Equation~\eqref{eq:decomposition}, with the above result, we can bound the risk the adversary induces with non-zero perturbations by the mass of points on which the classifier $\hypothesis_{\text{rob}}$ gives the good response but based on a low probability of success, \ie~with small confidence

\begin{equation}
\label{eq:premiseGeneralizationRandomizedSmoothing}
    \advRiskzero(\probmap) \leq \proba_{(\inputelement,y)\sim \groundDistrib} \left[ \hypothesis_{\text{rob}}(\inputelement)=y \emph{ and } \probmap(\inputelement)_{(1)} < \probmap(\inputelement)_{(2)} +2 \epsilon \right]. 
\end{equation}

This means that the only points on which the adversary may induce misclassification are the points on which $\probmap$ already has a high risk. Once more, this says something fundamental about the behavior of robust randomized classifiers. On undefended models, the adversary could change the decision on any point it wanted; now it is limited to changing points on which the classifier is already inaccurate. This considerably mitigates the threat model we should consider. Furthermore, for any deterministic classifier designed as in Equation~\eqref{eq:modeandRandomizedSmoothing}, we can also bound the maximal loss of accuracy under attack the classifier may suffer. This bound may, however, be harder to evaluate since it now depends on both the classifier and the dataset distribution. The classifier we define in Equation~\eqref{eq:modeandRandomizedSmoothing} and the mode preservation property of $\probmap$ are closely related to provable defenses based on randomized smoothing. The core idea of randomized smoothing is to take a hypothesis $\hypothesis$ and to build a robust classifier that writes 
\begin{equation}
    c_{rob}: \inputelement \mapsto \argmax\limits_{k \in [K]}\proba_{\vectorsym{z} \sim \mathcal{N}\left(0,\sigma^2 I\right)}\left[\hypothesis(\inputelement+\vectorsym{z}) = k\right]\equationspace.
\end{equation}

From a probabilistic point of view, for any input $\inputelement$, randomized smoothing amounts to output the most probable class of the probability measure $\probmap(\inputelement) \equaldef \hypothesis \# \mathcal{N}\left(\inputelement,\sigma^2 I\right)$.
Hence, randomized smoothing uses the mode preservation property of $\probmap$ to build a provably robust (deterministic) classifier. Therefore, the above results (Proposition~\ref{prop:modepreservationforTV} and Equation~\ref{eq:premiseGeneralizationRandomizedSmoothing}) also hold for provable defenses based on randomized smoothing. Studying randomized smoothing from our point of view could give an interesting new perspective on that method. So far no results have been published on the generalisation gap of this defense in the adversarial setting. We could devise generalization bounds by similarity with our analysis. 
Furthermore, the probabilistic interpretation stresses that randomized smoothing is somewhat restrictive since it only considers probability measures which are the expectation on a simple noise injection scheme.
The mode preservation property explains the behavior of randomized smoothing, but also presents fundamental properties of randomized defenses that could be used to construct more general defense schemes.

\section{Numerical validations: Gaussian Noise and $\ell_2$ adversary}
\label{section::Experiments}

To illustrate our findings, we train randomized neural networks with Gaussian pre-processing during training and inference on CIFAR-10 and CIFAR-100. Based on this randomized classifier, we study the impact of randomization on the standard accuracy of the network, and observe the theoretical trade-off between accuracy and robustness.

\subsection{Architecture and training procedure}
All the neural networks we use in this section are WideResNets~\cite{ZagoruykoK16} with $28$ layers, a widen factor of $10$, a dropout factor of $0.3$ and LeakyRelu activation with a $0.1$ slope. To train an undefended standard classifier we use the following hyper-parameters. 
        \begin{itemize}
            \item \textit{Number of Epochs:} 200
            \item \textit{Batch size:} 400
            \item \textit{Loss function:} Cross Entropy Loss
            \item \textit{Optimizer :} Stochastic gradient descent algorithm with momentum $0.9$, weight decay of $2\times10^{-4}$ and a learning rate that decreases during the training as follows: 
            \begin{align}
                lr = \left\{
                        \begin{matrix}
                        &0.1 & \text{if} & 0 &\leq & \text{epoch} & <& 60\\
                        &0.02 & \text{if} & 60 &\leq & \text{epoch} & <& 120\\
                        &0.004 & \text{if} & 120 &\leq & \text{epoch} & <& 160\\
                        &0.0008 & \text{if} & 160 &\leq & \text{epoch} & <& 200.\\
                        \end{matrix} \notag
                        \right.
            \end{align}
        \end{itemize}

To transform these standard networks into randomized classifiers, we inject noise drawn from Gaussian distributions, each with various standard deviations directly on the image before passing it through the network. Both during training and test, for computational efficiency, we evaluate the performance of the the algorithm over a single run for every images; hence no Monte Carlo estimator is used. However, in practice, the test-time accuracy is stable when evaluated over the entire test dataset.

\subsection{Results}
 
Figures~\ref{fig:GaussNoiseaccuracy} and~\ref{fig:GaussNoiseBound} show the accuracy and the minimum level of accuracy under attack of our randomized neural network for several levels of injected noise. We can see (Figure~\ref{fig:GaussNoiseaccuracy}) that the precision decreases as the noise intensity grows. In that sense, the noise must be calibrated to preserve both accuracy and robustness against adversarial attacks. This is to be expected, because the greater the entropy of the classifier, the less precise it gets.

 \begin{figure}[!ht]
\centering
    \includegraphics[width=\textwidth]{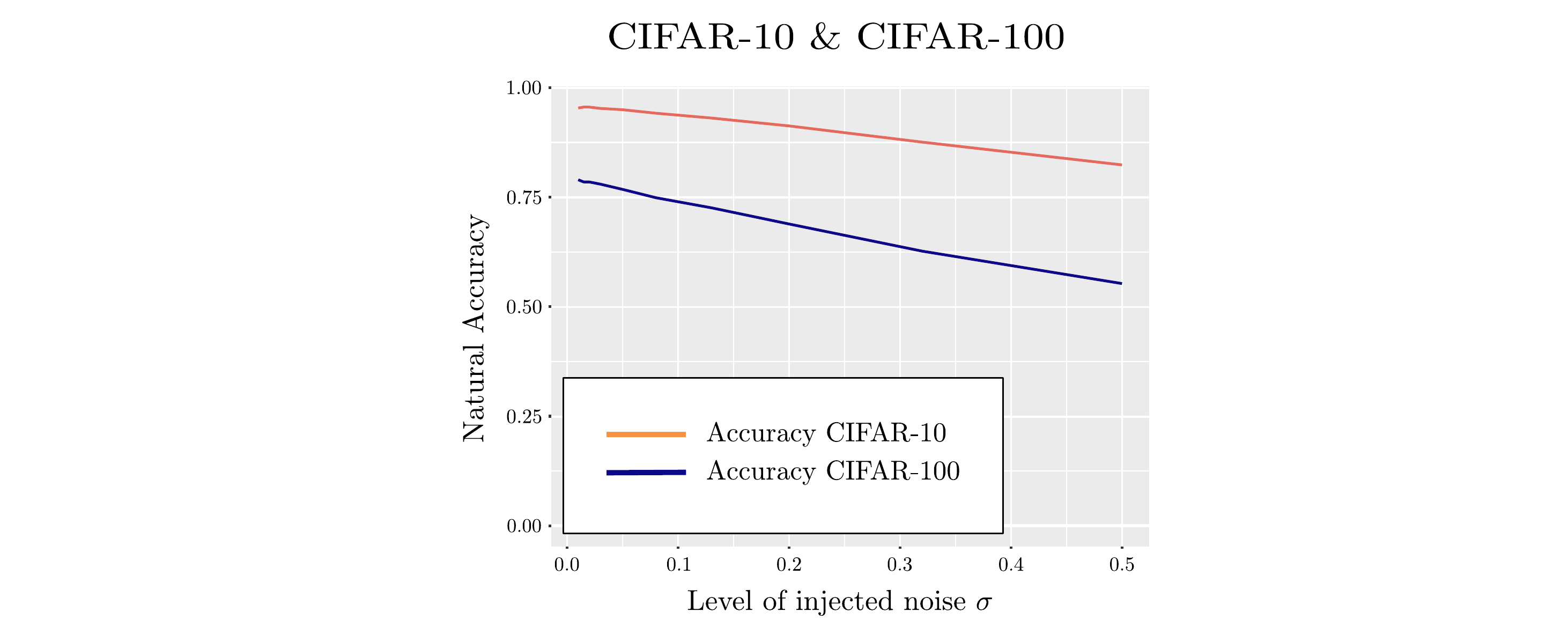}
\caption{Impact of the standard deviation of the Gausian noise on accuracy in a randomized model on CIFAR-10 and CIFAR-100 dataset.}
\label{fig:GaussNoiseaccuracy}
\end{figure}

Furthermore, when injecting Gaussian noise as a defense mechanism, the resulting randomized network $\probmap$ is both $(\alpha_2, \frac{(\alpha_2)^2}{2 \sigma})$-robust \wrt~$D_1$ and $(\alpha_2,2 \Phi\left( \frac{\alpha_2}{2 \sigma} \right) - 1)$-robust \wrt~$D_{TV}$ against $\ell_2$ adversaries. Therefore thanks to Theorems~\ref{th:TVboundRisk} and~\ref{th:RenyiboundRisk} we have that
\begin{align}
    \advRisk(\probmap; \alpha_2) - \Risk(\probmap) &\leq 2 \Phi\left( \frac{\alpha_2}{2 \sigma} \right) - 1, \text{\emph{ and}} \label{eq:boundTVriskgap}\\
    \advRisk(\probmap; \alpha_2) - \Risk(\probmap) &\leq 1-e^{-\frac{(\alpha_2)^2}{2 \sigma}} \expect_{\inputelement \sim \mathcal{D}_{\mid \inputspace}}\left[e^{-H(\probmap(\inputelement))}\right].\label{eq:boundRenyiriskgap}
\end{align}

\begin{figure}[!ht]
\centering
    \includegraphics[width=\textwidth]{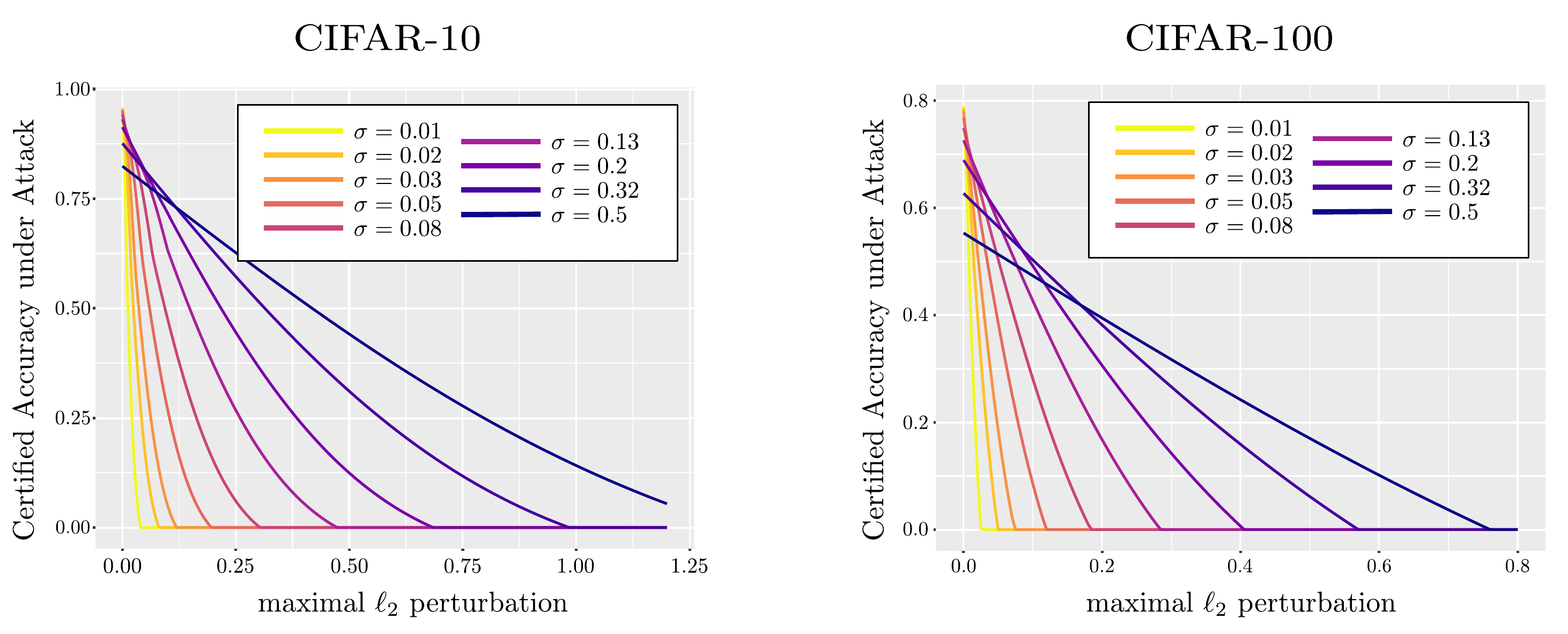}
\caption{Guaranteed accuracy of different randomized models with Gaussian noise given the $\ell_2$ norm of the adversarial perturbations.}
\label{fig:GaussNoiseBound}
\end{figure}

Figure~\ref{fig:GaussNoiseBound} illustrates the theoretical lower bound on accuracy under attack  (based on the minimum gap between Equations~\eqref{eq:boundTVriskgap} and~\eqref{eq:boundRenyiriskgap}) for different standard deviations. The term in entropy has been estimated using a Monte Carlo method with $10^4$ simulations. The trade-off between accuracy and robustness appears with respect to the noise intensity. With small noises, the accuracy is high, but the guaranteed accuracy drops fast with respect to the magnitude of the adversarial perturbation. Conversely, with bigger noises, the accuracy is lower but decreases slowly with respect to the magnitude of the adversarial perturbation. Overall, we get strong accuracy guarantees against small adversarial perturbations, but when the perturbation is bigger than $0.5$ on CIFAR-10 (resp. $0.3$ on CIFAR-100, the guarantees are still not sufficient).


\section{Lesson learned and future work}
\label{section::conclusion}

This paper brings new contributions to the theory of robustness to adversarial attacks. We provided an in depth analysis of randomized classifier, demonstrating their interest to defend against adversarial attacks. We first defined a notion of robustness for randomized classifiers using probability metrics/divergences, namely the total variation distance and the Renyi divergence. Second, we demonstrated that when a randomized classifier complies with this definition of robustness, we can bound their loss of accuracy under attack. We also studied the generalization properties of this class of functions and gave results indicating that robust randomized classifiers can generalize. Finally, we showed that randomized classifiers have a mode preservation property. This presents a fundamental property of randomized defenses that can be used to explain randomized smoothing from a probabilistic point of view.
To support our theoretical findings we presented a simple yet efficient scheme for building robust randomized classifiers. We show that Gaussian noise injection can provide principled robustness against $\ell_2$ adversarial attacks. We ran a set of experiments on CIFAR-10 and CIFAR-100 using Gaussian noise injection with advanced neural network architectures to build accurate models with controlled loss of accuracy under attack.

Future work will focus on studying the combination of randomization with more sophisticated defenses and on devising new tight bounds on the adversarial generalization and the adversarial risk gap of randomized classifiers. Based on the connections we established we randomized smoothing in Section~\ref{sec:modepreservationendRS}, we will also aim at devising bounds on the gap between the standard and adversarial risks for this defense. Another interesting direction would be to show that the classifiers based on randomized smoothing have a generalization gap similar to the classes of randomized classifiers we studied. 

\clearpage

\appendix
\section{Discussion on the metric/divergence one should consider}
\label{appendix::discussion}

As mentioned earlier in this paper, the choice of the metric/divergence is crucial as it characterizes the notion of adversarial robustness we are examining. We focus on the total variation distance and Renyi divergence, but the question of whether these metrics/divergences are more appropriate than others remains open. It should be noted, however, that our definition of robustness is monotonous depending on the metric/divergence we use.

\begin{proposition}[Monotonicity of the robustness]
\label{th::PropimpliesRobustness}
Let $\probmap$ be a randomized classifier, and let  $D$ and $D'$ be two divergences/metrics on $\probset(\outputspace)$.
If there exists a non decreasing function $ f: \R \to \R$ such that  $\forall \rho ,\rho' \in \probset(\outputspace)$, $D(\rho  , \rho') \leq f(D'(\rho  , \rho')) $, then the following assertion holds. 
$$\probmap \text{ is } (\alpha_p, \epsilon)\textnormal{-robust \wrt~}D' \implies \probmap \text{ is } (\alpha_p, f(\epsilon))\text{-robust \wrt~} D.$$
\end{proposition}
 The proof straightforwardly comes from the definition of robustness.
\begin{proof}
Let us consider $\probmap$ a randomized classifier $(\alpha_p, \epsilon)$-robust \wrt~$ D'$. Then for any $\inputelement \sim \groundDistrib$, and $\perturb ~\st~ \norm{\perturb}_p \leq \alpha_p$, since $f$ is non decreasing, we have $$D(\probmap(\inputelement),\probmap(\inputelement +\perturb)) \leq f\left(D'(\probmap(\inputelement),\probmap(\inputelement +\perturb))\right) \leq f\left(\epsilon\right).$$ 
Then $\probmap$ is  $(\alpha_p, f(\epsilon))$-robust \wrt~$D$ which concludes the proof.
\end{proof}

The above result suggests that the different notions of robustness we might conceive are more related than they appear. Here are some of the most classical divergences used in machine learning. Let $\rho,\rho',\nu$ three measures in $\probset(\outputspace)$. We denotes $g$ and $g'$ the probability density functions of $\rho$ and $\rho'$ with respect to $\nu$. Then we can define the \emph{Wasserstein distance} as follows 
\begin{equation}
        D_{W}(\rho  , \rho' ) \equaldef \inf\int_{\mathcal{Y}^{2}} \dist\left( y, y'\right) d\pi(y,y'),
\end{equation}
where $\dist$ is some ground distance on $\outputspace$, and the infimum is taken over all joint distributions $\pi$ in $\probset\left( \mathcal{Y}\times\mathcal{Y} \right)$ with marginals $\rho$ and $\rho'$.  

\begin{remark}
In transportation theory, the Wasserstein distance is solution of the Monge-Kantorovich problem with the cost function $c(y,y') = \dist(y,y')$. Then, the definitions of total variation and Wasserstein distance match when we use the trivial distance $\dist(y,y') = \mathds{1}\{y \neq y'\}$. 
\end{remark}

We also define respectively the \emph{Hellinger distance} and the \emph{Separation distance} as follows.
\begin{align}
    & D_{H}(\rho  , \rho'):= \left[ \int_{\mathcal{Y}} \left(\sqrt{g} - \sqrt{g'} \right)^{2} d \nu \right]^{1/2}.\\
    & D_{S}(\rho  , \rho'):= \sup\limits_{y \in \mathcal{Y}} \left( 1 - \frac{g(y)}{g'(y)} \right). 
\end{align}

\begin{figure}
    \centering
    \includegraphics[width=0.9\textwidth]{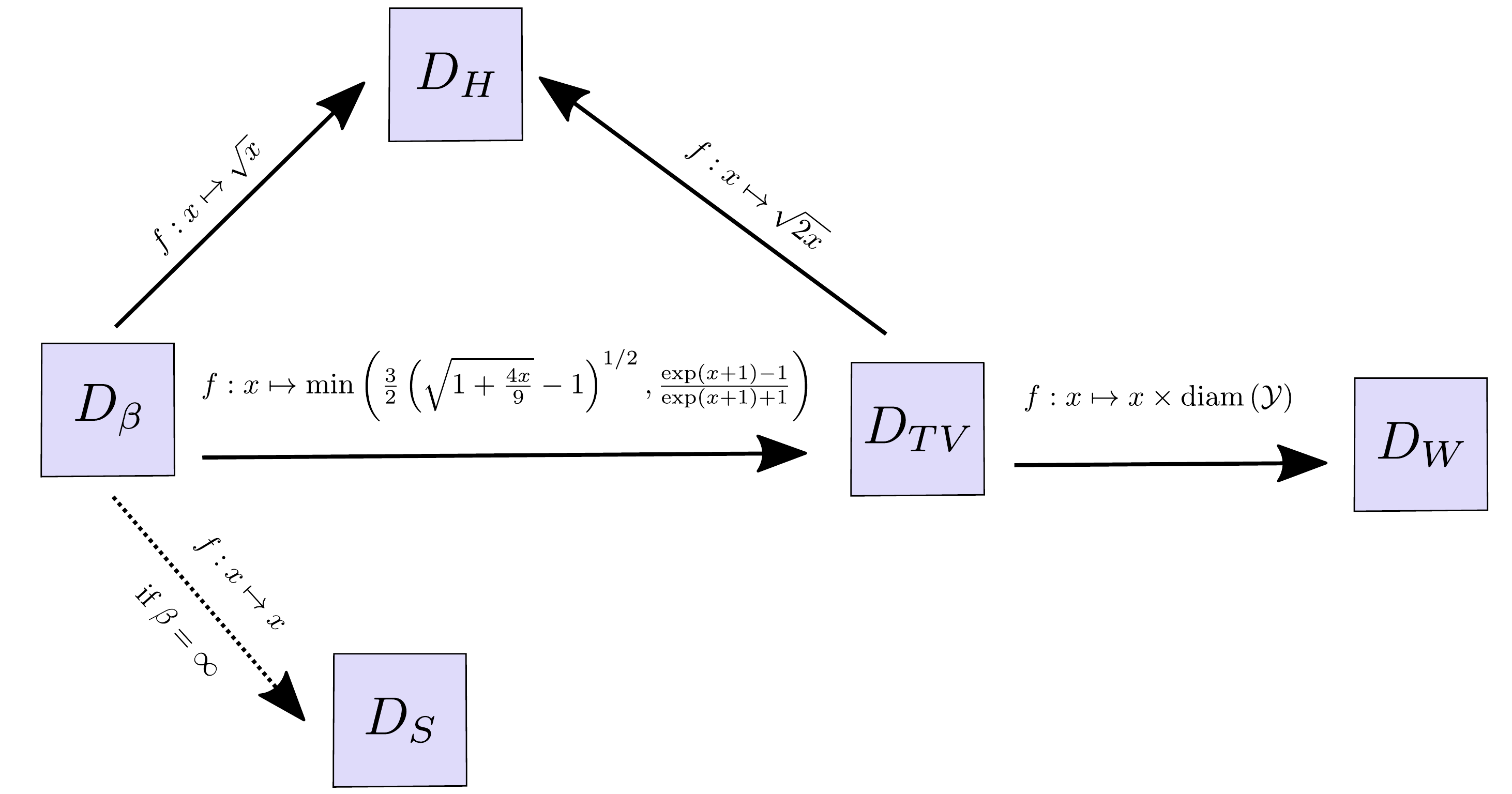}
    \caption{Summary of the relations between the different robustness notions from Propositions~\ref{prop:LinkTV-Wass} and~\ref{prop:LinkRenyiOthers}. }
    \label{fig:Proposition67}
\end{figure}

If we take any of the above metrics/divergences to instantiate a notion of adversarial robustness we might get very different semantics for them. However, we can show that any of these definitions can be covered -- with respect to Proposition~\ref{th::PropimpliesRobustness} -- either by the Renyi or the total variation robustness. 
Figure~\ref{fig:Proposition67} summarizes the links we can make between all these different definitions of robustness, and Propositions~\ref{prop:LinkTV-Wass} and~\ref{prop:LinkRenyiOthers} present the associated results. We can see that the total variation distance and the Renyi divergence are both central since they can cover any of the other robustness notions. This does not mean that they are more appropriate than the others, but at least they are general enough to cover a wide range of possible definitions.

\begin{proposition}
\label{prop:LinkTV-Wass}
Let $\probmap$ be a randomized classifier. If $\probmap$ is $(\alpha_p, \epsilon)$-robust \wrt~$D_{TV}$ then the following assertions hold. 
\begin{itemize}
    \item $\probmap$ is $\left(\alpha_p, \epsilon \times \diam\left( \outputspace \right) \right)$-robust \wrt~$D_{W}$, where $\diam\left( \outputspace \right) \equaldef \max\limits_{y,y' \in \outputspace} \dist(y,y')$.
     \item $\probmap$ is $\left(\alpha_p, \sqrt{2 \epsilon} \right)$-robust \wrt~$D_{H}$.
\end{itemize} 
\end{proposition}
\begin{proof}
Let us consider $\rho$ and $\rho' \in \probset\left(\outputspace\right)$. Thanks to~\cite{AGibbsMetrics2002} we have
\begin{itemize}
    \item $D_W(\rho,\rho') \leq \diam(\mathcal{Y}) D_{TV}(\rho,\rho')$.
    \item $D_H(\rho,\rho') \leq  \sqrt{2 D_{TV}(\rho,\rho')}$.
\end{itemize}
Hence, by using Proposition~\ref{th::PropimpliesRobustness} respectively with $f: x \mapsto \diam(\mathcal{Y}) x$ and  $f: x \mapsto \sqrt{2x}$ we get the expected results.
\end{proof}

\begin{proposition}
\label{prop:LinkRenyiOthers}
Let $\probmap$ be a randomized classifier. If $\probmap$ is $(\alpha_p, \epsilon)$-robust \wrt~$D_{\beta}$ then the following assertions hold. 
\begin{itemize}
\item $\probmap$ is $(\alpha_p, \epsilon')$-robust \wrt~$D_{TV}$ with $\epsilon' = \min \left(\frac{3}{2}\left(\sqrt{1 + \frac{4\epsilon}{9}} - 1\right)^{1/2}, \frac{\exp(\epsilon +1) -1}{\exp(\epsilon +1) +1}\right)$.
\item $\probmap$ is $(\alpha_p, \sqrt{\epsilon})$-robust \wrt~$D_H$.
\item If $\beta =\infty$, then $\probmap$ is $(\alpha_p, \epsilon)$ robust \wrt~$D_S$.
\end{itemize}
\end{proposition}

\begin{proof}
1) First, let us suppose that $\beta \geq 1$. Thanks to Proposition~\ref{prop:Inequality-TV-Renyi} and to~\cite{AGibbsMetrics2002}, for any $\rho, \rho' \in \probset\left( \outputspace \right)$ we have
\begin{itemize}
\item $D_H(\rho,\rho') \leq \sqrt{ D_{1}(\rho,\rho') } \leq \sqrt{ D_{\beta}(\rho,\rho') } $ \quad \text{(see~\cite{AGibbsMetrics2002})}.
\item $D_{TV}(\rho,\rho') \leq \min \left(\frac{3}{2}\left(\sqrt{1 + \frac{4D_{\beta}(\rho,\rho') }{9}} - 1\right)^{1/2}, \frac{\exp(D_{\beta}(\rho,\rho')  +1) -1}{\exp(D_{\beta}(\rho,\rho')  +1) +1}\right)$ \text{(Prop.~\ref{prop:Inequality-TV-Renyi})}.
\end{itemize}
Hence, by using Proposition~\ref{th::PropimpliesRobustness}, as above, we get the expected results.  

\noindent 2) Now let us suppose that $\beta = \infty$. By definition of the supremum divergence, we have $$
D_{\infty}(\rho,\rho') = \sup_{B \subset \mathcal{Y}} \ \left|\ln\frac{\rho(B)}{\rho'(B)}\right|.$$
Furthermore, note that the function $x \mapsto  1-x - \left|\ln(x)\right|$ is negative on $\mathbb{R}$, therefore for any $y \in \mathcal{Y}$ one has $$1-\frac{\rho(y)}{\rho'(y)} \leq \left|\ln\frac{\rho(y)}{\rho'(y)}\right|.$$ 
Since the above inequality is true for any $y \in \outputspace$, we have $$D_S\left(\rho,\rho'\right) = \sup_{y \in \mathcal{Y}}\left(1-\frac{\rho(y)}{\rho'(y)} \right) \leq \sup_{y \in \mathcal{Y}}\left|\ln\frac{\rho(y)}{\rho'(y)}\right| \leq \sup_{B \subset \mathcal{Y} }\left|\ln\frac{\rho(B)}{\rho'(B)}\right| = D_{\infty}(\rho,\rho').$$ 
Finally, by using Proposition~\ref{th::PropimpliesRobustness} with $f: x \mapsto x$ we get the expected results.
\end{proof}

\newpage
\bibliography{biblio}

\begin{thebibliography}{10}

\bibitem{obfuscated-gradients}
A.~Athalye, N.~Carlini, and D.~Wagner.
\newblock Obfuscated gradients give a false sense of security: Circumventing
  defenses to adversarial examples.
\newblock In {\em Proceedings of the 35th International Conference on Machine
  Learning, {ICML} 2018}, July 2018.

\bibitem{awasthi2020adversarial}
P.~Awasthi, N.~Frank, and M.~Mohri.
\newblock Adversarial learning guarantees for linear hypotheses and neural
  networks.
\newblock {\em International Conference on Machine Learning}, 2020.

\bibitem{bartlett2002rademacher}
P.~L. Bartlett and S.~Mendelson.
\newblock Rademacher and gaussian complexities: Risk bounds and structural
  results.
\newblock {\em Journal of Machine Learning Research}, 3:463--482, 2002.

\bibitem{ben2009robust}
A.~Ben-Tal, L.~El~Ghaoui, and A.~Nemirovski.
\newblock {\em Robust optimization}, volume~28.
\newblock Princeton University Press, 2009.

\bibitem{biggio2013evasion}
B.~Biggio, I.~Corona, D.~Maiorca, B.~Nelson, N.~{\v{S}}rndi{\'c}, P.~Laskov,
  G.~Giacinto, and F.~Roli.
\newblock Evasion attacks against machine learning at test time.
\newblock In {\em Joint European conference on machine learning and knowledge
  discovery in databases}, pages 387--402. Springer, 2013.

\bibitem{carlini2017adversarial}
N.~Carlini and D.~Wagner.
\newblock Adversarial examples are not easily detected: Bypassing ten detection
  methods.
\newblock In {\em Proceedings of the 10th ACM Workshop on Artificial
  Intelligence and Security}, pages 3--14, 2017.

\bibitem{ChapR04}
F.~Chapeau-Blondeau and D.~Rousseau.
\newblock Noise-enhanced performance for an optimal bayesian estimator.
\newblock {\em IEEE Transactions on Signal Processing}, 52(5):1327--1334, 2004.

\bibitem{chen2018ead}
P.-Y. Chen, Y.~Sharma, H.~Zhang, J.~Yi, and C.-J. Hsieh.
\newblock Ead: Elastic-net attacks to deep neural networks via adversarial
  examples.
\newblock In {\em AAAI}, 2018.

\bibitem{KolterRandomizedSmoothing}
J.~M. Cohen, E.~Rosenfeld, and J.~Z. Kolter.
\newblock Certified adversarial robustness via randomized smoothing.
\newblock In {\em International Conference on Machine Learning}, 2019.

\bibitem{cover2012elements}
T.~M. Cover and J.~A. Thomas.
\newblock {\em Elements of information theory}.
\newblock John Wiley \& Sons, 2012.

\bibitem{croce2020reliable}
F.~Croce and M.~Hein.
\newblock Reliable evaluation of adversarial robustness with an ensemble of
  diverse parameter-free attacks.
\newblock In {\em International Conference on Machine Learning}, 2020.

\bibitem{dalvi2004adversarial}
N.~Dalvi, P.~Domingos, S.~Sanghai, and D.~Verma.
\newblock Adversarial classification.
\newblock In {\em Proceedings of the tenth ACM SIGKDD international conference
  on Knowledge discovery and data mining}, pages 99--108, 2004.

\bibitem{pruningDefenseICLR2018}
G.~S. Dhillon, K.~Azizzadenesheli, J.~D. Bernstein, J.~Kossaifi, A.~Khanna,
  Z.~C. Lipton, and A.~Anandkumar.
\newblock Stochastic activation pruning for robust adversarial defense.
\newblock In {\em International Conference on Learning Representations}, 2018.

\bibitem{dong2019gaussian}
J.~Dong, A.~Roth, and W.~J. Su.
\newblock Gaussian differential privacy.
\newblock {\em arXiv preprint arXiv:1905.02383}, 2019.

\bibitem{AGibbsMetrics2002}
A.~L. Gibbs and F.~E. Su.
\newblock On choosing and bounding probability metrics.
\newblock {\em International Statistical Review / Revue Internationale de
  Statistique}, 70(3):419--435, 2002.

\bibitem{5605338}
G.~L. {Gilardoni}.
\newblock On pinsker's and vajda's type inequalities for
  csiszár's$f$-divergences.
\newblock {\em IEEE Transactions on Information Theory}, 56(11):5377--5386,
  2010.

\bibitem{globerson2006nightmare}
A.~Globerson and S.~Roweis.
\newblock Nightmare at test time: robust learning by feature deletion.
\newblock In {\em Proceedings of the 23rd international conference on Machine
  learning}, pages 353--360, 2006.

\bibitem{goodfellow2014explaining}
I.~Goodfellow, J.~Shlens, and C.~Szegedy.
\newblock Explaining and harnessing adversarial examples.
\newblock In {\em International Conference on Learning Representations}, 2015.

\bibitem{grandvalet1997noise}
Y.~Grandvalet, S.~Canu, and S.~Boucheron.
\newblock Noise injection: Theoretical prospects.
\newblock {\em Neural Computation}, 9(5):1093--1108, 1997.

\bibitem{he2017adversarial}
W.~He, J.~Wei, X.~Chen, N.~Carlini, and D.~Song.
\newblock Adversarial example defense: Ensembles of weak defenses are not
  strong.
\newblock In {\em 11th $\{$USENIX$\}$ Workshop on Offensive Technologies
  ($\{$WOOT$\}$ 17)}, 2017.

\bibitem{hu2019new}
S.~Hu, T.~Yu, C.~Guo, W.-L. Chao, and K.~Q. Weinberger.
\newblock A new defense against adversarial images: Turning a weakness into a
  strength.
\newblock In {\em Advances in Neural Information Processing Systems}, pages
  1635--1646, 2019.

\bibitem{10.5555/3327546.3327734}
S.~Jetley, N.~A. Lord, and P.~H. Torr.
\newblock With friends like these, who needs adversaries?
\newblock In {\em Proceedings of the 32nd International Conference on Neural
  Information Processing Systems}, NIPS’18, page 10772–10782, Red Hook, NY,
  USA, 2018. Curran Associates Inc.

\bibitem{kearns1993learning}
M.~Kearns and M.~Li.
\newblock Learning in the presence of malicious errors.
\newblock {\em SIAM Journal on Computing}, 22(4):807--837, 1993.

\bibitem{kearns1994toward}
M.~J. Kearns, R.~E. Schapire, and L.~M. Sellie.
\newblock Toward efficient agnostic learning.
\newblock {\em Machine Learning}, 17(2-3):115--141, 1994.

\bibitem{khim2018adversarial}
J.~Khim and P.-L. Loh.
\newblock Adversarial risk bounds for binary classification via function
  transformation.
\newblock {\em arXiv preprint arXiv:1810.09519}, 2, 2018.

\bibitem{krizhevsky2009learning}
A.~Krizhevsky and G.~Hinton.
\newblock Learning multiple layers of features from tiny images.
\newblock Technical report, Citeseer, 2009.

\bibitem{langlois2014gghlite}
A.~Langlois, D.~Stehl{\'e}, and R.~Steinfeld.
\newblock Gghlite: More efficient multilinear maps from ideal lattices.
\newblock In {\em Annual International Conference on the Theory and
  Applications of Cryptographic Techniques}, pages 239--256. Springer, 2014.

\bibitem{lecuyer2019certified}
M.~Lecuyer, V.~Atlidakis, R.~Geambasu, D.~Hsu, and S.~Jana.
\newblock Certified robustness to adversarial examples with differential
  privacy.
\newblock In {\em 2019 IEEE Symposium on Security and Privacy (SP)}, pages
  656--672. IEEE, 2019.

\bibitem{li2019certified}
B.~Li, C.~Chen, W.~Wang, and L.~Carin.
\newblock Certified adversarial robustness with additive noise.
\newblock In {\em Advances in Neural Information Processing Systems}, pages
  9464--9474, 2019.

\bibitem{Xuang2018}
X.~Liu, M.~Cheng, H.~Zhang, and C.-J. Hsieh.
\newblock Towards robust neural networks via random self-ensemble.
\newblock In {\em European Conference on Computer Vision}, pages 381--397.
  Springer, 2018.

\bibitem{lowd2005adversarial}
D.~Lowd and C.~Meek.
\newblock Adversarial learning.
\newblock In {\em Proceedings of the eleventh ACM SIGKDD international
  conference on Knowledge discovery in data mining}, pages 641--647, 2005.

\bibitem{madry2017towards}
A.~Madry, A.~Makelov, L.~Schmidt, D.~Tsipras, and A.~Vladu.
\newblock Towards deep learning models resistant to adversarial attacks.
\newblock In {\em International Conference on Learning Representations}, 2018.

\bibitem{metzen2017detecting}
J.~H. Metzen, T.~Genewein, V.~Fischer, and B.~Bischoff.
\newblock On detecting adversarial perturbations.
\newblock In {\em Proceedings of 5th International Conference on Learning
  Representations (ICLR)}, 2017.

\bibitem{MitaK98}
S.~Mitaim and B.~Kosko.
\newblock Adaptive stochastic resonance.
\newblock {\em Proceedings of the IEEE}, 86(11):2152--2183, 1998.

\bibitem{mohri2018foundations}
M.~Mohri, A.~Rostamizadeh, and A.~Talwalkar.
\newblock {\em Foundations of machine learning}.
\newblock 2018.

\bibitem{papernot2016distillation}
N.~Papernot, P.~McDaniel, X.~Wu, S.~Jha, and A.~Swami.
\newblock Distillation as a defense to adversarial perturbations against deep
  neural networks.
\newblock In {\em 2016 IEEE Symposium on Security and Privacy (SP)}, pages
  582--597. IEEE, 2016.

\bibitem{Perez2017TheEO}
L.~Perez and J.~Wang.
\newblock The effectiveness of data augmentation in image classification using
  deep learning.
\newblock {\em arXiv preprint arXiv:1712.04621}, 2017.

\bibitem{peyre2019computational}
G.~Peyr{\'e}, M.~Cuturi, et~al.
\newblock Computational optimal transport: With applications to data science.
\newblock {\em Foundations and Trends{\textregistered} in Machine Learning},
  11(5-6):355--607, 2019.

\bibitem{pinot2019theoretical}
P.~Rafael, M.~Laurent, A.~Alexandre, K.~Hisashi, Y.~Florian, G.-P. C{\'e}dric,
  and A.~Jamal.
\newblock Theoretical evidence for adversarial robustness through
  randomization.
\newblock In {\em Advances in Neural Information Processing Systems}, pages
  11838--11848, 2019.

\bibitem{rakin2018parametricnoiseinjection}
A.~S. Rakin, Z.~He, and D.~Fan.
\newblock Parametric noise injection: Trainable randomness to improve deep
  neural network robustness against adversarial attack.
\newblock {\em arXiv preprint arXiv:1811.09310}, 2018.

\bibitem{renyi1961}
A.~R{\'e}nyi.
\newblock On measures of entropy and information.
\newblock Technical report, Hungarian Academy of Sciences Budapest Hungary,
  1961.

\bibitem{robert2007bayesian}
C.~Robert.
\newblock {\em The Bayesian choice: from decision-theoretic foundations to
  computational implementation}.
\newblock Springer Science \& Business Media, 2007.

\bibitem{salman2019provably}
H.~Salman, J.~Li, I.~Razenshteyn, P.~Zhang, H.~Zhang, S.~Bubeck, and G.~Yang.
\newblock Provably robust deep learning via adversarially trained smoothed
  classifiers.
\newblock In {\em Advances in Neural Information Processing Systems}, pages
  11289--11300, 2019.

\bibitem{schmidt2018adversarially}
L.~Schmidt, S.~Santurkar, D.~Tsipras, K.~Talwar, and A.~Madry.
\newblock Adversarially robust generalization requires more data.
\newblock In {\em Advances in Neural Information Processing Systems}, pages
  5014--5026, 2018.

\bibitem{shalev2014understanding}
S.~Shalev-Shwartz and S.~Ben-David.
\newblock {\em Understanding machine learning: From theory to algorithms}.
\newblock Cambridge university press, 2014.

\bibitem{sharif2016accessorize}
M.~Sharif, S.~Bhagavatula, L.~Bauer, and M.~K. Reiter.
\newblock Accessorize to a crime: Real and stealthy attacks on state-of-the-art
  face recognition.
\newblock In {\em Proceedings of the 2016 acm sigsac conference on computer and
  communications security}, pages 1528--1540, 2016.

\bibitem{simon2019first}
C.-J. Simon-Gabriel, Y.~Ollivier, L.~Bottou, B.~Sch{\"o}lkopf, and
  D.~Lopez-Paz.
\newblock First-order adversarial vulnerability of neural networks and input
  dimension.
\newblock In {\em International Conference on Machine Learning}, pages
  5809--5817, 2019.

\bibitem{sitawarin2018darts}
C.~Sitawarin, A.~N. Bhagoji, A.~Mosenia, M.~Chiang, and P.~Mittal.
\newblock Darts: Deceiving autonomous cars with toxic signs.
\newblock {\em arXiv preprint arXiv:1802.06430}, 2018.

\bibitem{su2018robustness}
D.~Su, H.~Zhang, H.~Chen, J.~Yi, P.-Y. Chen, and Y.~Gao.
\newblock Is robustness the cost of accuracy?--a comprehensive study on the
  robustness of 18 deep image classification models.
\newblock In {\em Proceedings of the European Conference on Computer Vision
  (ECCV)}, pages 631--648, 2018.

\bibitem{Szegedy2013IntriguingPO}
C.~Szegedy, W.~Zaremba, I.~Sutskever, J.~Bruna, D.~Erhan, I.~Goodfellow, and
  R.~Fergus.
\newblock Intriguing properties of neural networks.
\newblock In {\em International Conference on Learning Representations}, 2014.

\bibitem{tramer2020adaptive}
F.~Tramer, N.~Carlini, W.~Brendel, and A.~Madry.
\newblock On adaptive attacks to adversarial example defenses.
\newblock \emph{arXiv preprint arXiv:2002.08347}, 2020.

\bibitem{tsipras2018robustness}
D.~Tsipras, S.~Santurkar, L.~Engstrom, A.~Turner, and A.~Madry.
\newblock Robustness may be at odds with accuracy.
\newblock {\em International Conference on Learning Representation}, 2019.

\bibitem{Vajda1970}
I.~Vajda.
\newblock Note on discrimination information and variation.
\newblock {\em IEEE Trans. Inform. Theory}, 16(6):771--773, 1970.

\bibitem{van2000asymptotic}
A.~W. Van~der Vaart.
\newblock {\em Asymptotic statistics}, volume~3.
\newblock Cambridge university press, 2000.

\bibitem{6832827}
T.~{van Erven} and P.~{Harremos}.
\newblock R\'enyi divergence and kullback-leibler divergence.
\newblock {\em IEEE Transactions on Information Theory}, 60(7):3797--3820,
  2014.

\bibitem{NIPS2019_9070}
G.~Verma and A.~Swami.
\newblock Error correcting output codes improve probability estimation and
  adversarial robustness of deep neural networks.
\newblock In H.~Wallach, H.~Larochelle, A.~Beygelzimer, F.~d'~Alch\'{e}-Buc,
  E.~Fox, and R.~Garnett, editors, {\em Advances in Neural Information
  Processing Systems 32}, pages 8646--8656. Curran Associates, Inc., 2019.

\bibitem{villani2003topics}
C.~Villani.
\newblock {\em Topics in optimal transportation}.
\newblock Number~58. American Mathematical Soc., 2003.

\bibitem{xie2018mitigating}
C.~Xie, J.~Wang, Z.~Zhang, Z.~Ren, and A.~Yuille.
\newblock Mitigating adversarial effects through randomization.
\newblock In {\em International Conference on Learning Representations}, 2018.

\bibitem{Xie2017MitigatingAE}
C.~Xie, J.~Wang, Z.~Zhang, Z.~Ren, and A.~Yuille.
\newblock Mitigating adversarial effects through randomization.
\newblock In {\em International Conference on Learning Representations}, 2018.

\bibitem{xu2012robustness}
H.~Xu and S.~Mannor.
\newblock Robustness and generalization.
\newblock {\em Machine learning}, 86(3):391--423, 2012.

\bibitem{selfdrivingattack2020}
D.~Yao, Z.~Xi, Z.~Tianyi, C.~Chen, L.~Guannan, and K.~Miryung.
\newblock An analysis of adversarial attacks and defenses on autonomous driving
  models.
\newblock In {\em 18th Annual IEEE International Conference on Pervasive
  Computing and Communications}. IEEE, 2020.

\bibitem{yin2019rademacher}
D.~Yin, R.~Kannan, and P.~Bartlett.
\newblock Rademacher complexity for adversarially robust generalization.
\newblock In {\em International Conference on Machine Learning}, pages
  7085--7094, 2019.

\bibitem{ZagoruykoK16}
S.~Zagoruyko and N.~Komodakis.
\newblock Wide residual networks.
\newblock In {\em Proceedings of the British Machine Vision Conference (BMVC)},
  pages 87.1--87.12. BMVA Press, 2016.

\bibitem{zhang2019theoretically}
H.~Zhang, Y.~Yu, J.~Jiao, E.~P. Xing, L.~E. Ghaoui, and M.~I. Jordan.
\newblock Theoretically principled trade-off between robustness and accuracy.
\newblock {\em International conference on Machine Learning}, 2019.

\bibitem{ZozoA99}
S.~Zozor and P.-O. Amblard.
\newblock Stochastic resonance in discrete time nonlinear {AR}(1) models.
\newblock {\em IEEE transactions on Signal Processing}, 47(1):108--122, 1999.

\end{thebibliography}
\bibliographystyle{abbrv}
\newpage

\end{document}